\documentclass{article}

\usepackage{natbib}
\usepackage{smacros}

\usepackage[utf8]{inputenc} 
\usepackage[T1]{fontenc}    
\usepackage{hyperref}       
\usepackage{url}            
\usepackage{booktabs}       
\usepackage{amsmath,amsthm,amssymb,amsfonts}      
\usepackage{nicefrac}       
\usepackage{microtype}      
\usepackage{xcolor}         
\usepackage{enumitem}
\usepackage{mathtools}
\usepackage{todonotes}
\usepackage{algorithm}
\usepackage{algorithmic}
\usepackage{colortbl}
\usepackage{multirow}
\usepackage{graphicx}
\usepackage{subcaption}

\definecolor{cellbest}{RGB}{196,235,196}   
\definecolor{cellgood}{RGB}{255,245,190}   
\definecolor{cellbad}{RGB}{255,205,205}    
\definecolor{headerbg}{RGB}{240,240,240}

\newtheorem{assumption}{Assumption}

\newcommand{\calO}{\mathcal{O}}
\newcommand{\tildeO}{\Tilde{\mathcal{O}}}
\newcommand{\Prob}{\mathbb{P}}

\newcommand{\zhat}{\hat{z}}
\newcommand{\that}{\hat{\theta}}
\newcommand{\tstar}{\theta^{*}}

\newcommand{\best}[1]{\cellcolor{cellbest}\textbf{#1}}
\newcommand{\good}[1]{\cellcolor{cellgood}#1}
\newcommand{\bad}[1]{\cellcolor{cellbad}#1}

\title{Optimal Regret for Single Index Bandits}

\author{%
  Devdan Dey 
 \footnote{ Department of Computer Science and Engineering,
  Indian Institute of Technology Bombay.
  Email: \url{23d0365@iitb.ac.in}} \quad \quad
  Sujoy Bhore 
  \footnote{Department of Computer Science and Engineering,
  Indian Institute of Technology Bombay.
  Email: \url{sujoy@cse.iitb.ac.in}} \quad \quad
  Avishek Ghosh
 \footnote{ Department of Computer Science and Engineering,
  Indian Institute of Technology Bombay.
  Email:
  \url{avishek_ghosh@iitb.ac.in}}
}

\begin{document}

\date{}
\maketitle

\begin{abstract}
We study the \emph{single-index bandit} problem, 
where rewards depend on an unknown 
one-dimensional projection 
of high-dimensional contexts through an unknown reward function. This model extends linear and generalized linear bandits to a nonparametric setting, and is particularly relevant when the reward function is not known in advance. 
While optimal regret guarantees are known for monotone reward functions, 
the general non-monotone case remains poorly understood, 
with the best known bound being 
$\tilde{\mathcal{O}}(T^{3/4})$ (under standard boundedness and Lipschitz assumptions on the reward function \cite{kang2025single}).

We close this gap by establishing 
the optimal regret for general single-index bandits. We propose a simple two-phase algorithm, namely, Zoomed Single Index Bandit with Upper Confidence Bound (\texttt{ZoomSIB-UCB}), that first estimates the projection direction via a normalized Stein estimator, and then reduces the problem to a one-dimensional bandit using discretization and finally use UCB. This approach achieves a regret of $\tilde{\mathcal{O}}(T^{2/3})$, and improves significantly upon prior work without any additional assumptions. We also prove a matching minimax lower bound of $\tilde{\Omega}(T^{2/3})$, showing that the upper bound is essentially tight. Our upper and lower bounds together provide a sharp characterization of the regret in single-index bandits. Moreover, the empirical results further demonstrate the effectiveness and robustness of our approach.

\end{abstract}

\section{Introduction}
In a single index model (SIM) on data $(x,y) \in \mathbb{R}^{d}\times \mathbb{R}$, we let $y$ depend on the projection of $x$ onto an unknown parameter $\theta_*$ via an unknown function $f:\mathbb{R} \rightarrow \mathbb{R}$. Hence, learning in SIM is nonparametric since we need to estimate the parameter $\theta_*$ as well as the unknown function $f$. As an extension of the Generalized Linear Model, the SIM is a versatile statistical framework that has been applied across longitudinal data analysis, quantile regression, and econometrics, and has more recently attracted substantial attention from the theoretical deep learning community as a tool for evaluating the ability of neural networks to learn low-dimensional representations, in contrast to kernel methods (\cite{plan2016generalized,brillinger1982generalized,bruna2025survey}).

Learning with SIM remains an active area of research. Early foundational work by \cite{ichimura1993semiparametric} 
and \cite{hardle1993optimal} establish semiparametric efficiency bounds and 
optimal estimation rates for the SIM in the classical low-dimensional regime. 
\cite{hornik1989multilayer} and \cite{mccullagh1989generalized} connect the SIM to neural networks and generalized linear models respectively. In the statistical learning theory literature, \cite{plan2016generalized} 
and \cite{goldstein2018structured} study SIMs through the lens of empirical 
process theory and Gaussian width. More recently, the 
connection between SIMs and gradient descent on two-layer neural networks has 
attracted significant attention ( \cite{barak2022hidden,abbe2022merged}). 

Recently, there has been significant interest in understanding the single index model in the sequential framework. The multi-armed bandit is a canonical framework in modeling online (sequential) learning and is classically used for recommendation, clinical trials, targeted ads and hyperparameter tuning \cite{Lattimore_Szepesvari_2020,bubeck2012regret}. Traditionally, in order to incorporate contextual information, a variant of multi armed bandit framework, namely contextual linear bandit \cite{chu-contextual,yadkori,dani2008stochastic,li2010contextual} is studied where the reward is assumed to be a linear function of the context. However, in many applications like click through rate in recommendation systems the linear model fails to capture the reward. To address this, \cite{filippi-glm,li2017provably,zhang2025generalized,faury2020improved,abeille2021instance} study the generalized linear bandit framework where the reward is assumed to depend on the context via a generalized linear model. Formally, the mean reward at time $t$ is given by $f(x^T \theta_*)$ where the reward function $f(.)$ is known in advance.

Although such a model captures a lot of practical application, assumption about the knowledge of $f(\cdot)$ is often unrealistic and it has been shown in \cite{ghosh2017misspecified,bogunovic2021misspecified} that misspecification can lead to linear regret. To address this issue, \cite{kang2025single} introduce single index bandit, where the reward function is not known in advance. When the reward function is monotone, the authors of \cite{kang2025single} show that exploiting the Stein estimator, an optimal regret scaling of $\mathcal{O}(\sqrt{T})$ is possible where $T$ is the learning horizon. However, when the reward function is non-monotone, \cite{kang2025single} assumes bounded and Lipschitz structure on the reward function and use kernel based estimator to learn the (unknown) function. Using an explore then commit type algorithm resulting in a regret of $\tildeO(T^{3/4})$.

There has been a few recent works on single index learning in bandits. For example, in \cite{arya2026kernel}, the authors assume that the reward function comes from a Reproducing Kernel Hilbert Space (RKHS) and exploiting the reproducing property as well as representer theorems, the authors show that in some cases $\mathcal{O}(\sqrt{T})$ is possible. Moreover, in \cite{ma2025nonparametric} reward functions coming from the intersection of smoothness and monotone function class is studied. The paper also proves matching lower bounds for such smooth monotone functions.

Summarizing the above results, we see that for monotone single index bandits, the optimal regret is well understood in literature. However, for general non-monotone functions (with bounded ness and Lipschitz structure), apart from the $\tildeO(T^{3/4})$ regret in \cite{kang2025single}, no other results are known to the best of our knowledge. This forms the central question of our work: 
\begin{center}
    \textit{What is the optimal regret scaling for general single index bandit?}
\end{center}
In this paper, we answer the above question. We keep the assumptions identical to that of \cite{kang2025single} (i.e., the reward function is bounded and Lipschitz) and obtain an improved regret bound of $\tildeO(T^{2/3})$. We leverage tools from the Lipschitz bandit literature \cite{kleinberg2004nearly,kleinberg2008multi,agrawal1995continuum,auer2007improved,slivkins2019introduction} and propose an algorithm that uses discretization and an appropriately instantiated Upper Confidence Bound (UCB) algorithm. Moreover, we prove a minimax lower bound of $\tilde{\Omega}(T^{2/3})$ implying that no bandit algorithm can incur better regret in the single index setting. Combining these two results, we answer the aforementioned question that the correct regret scaling for general (non-monotone) single index bandit is $\Theta(T^{2/3})$ (ignoring log factors). We now summarize our contributions.

\subsection{Summary of contributions}
\emph{Improved Regret:} We propose a learning algorithm namely Zoomed Single Index Bandit with Upper Confidence Bound (\texttt{ZoomSIB-UCB}) in Section~\ref{sec:algo}. On a high level, the algorithm breaks the single index learning to a sequence of one-dimensional bandit problems. In the first phase of \texttt{ZoomSIB-UCB}, we use the Stein estimator \cite{stein1981estimation} with proper normalization and leverage the Lipschitz property of the reward function to discretize the reward space with appropriate resolution into $N$ bins. In the second phase, we play UCB over $N$ bins to choose the arm and observe the associated reward. It turns out that $N \sim T^{1/3}$ is sufficient and as a result the regret of \texttt{ZoomSIB-UCB} is $\tildeO(T^{2/3})$. This may be seen as a direct improvement over \cite{kang2025single}. 

\emph{Matching Lower Bound:} We show that the $\tildeO(T^{2/3})$ regret scaling is unavoidable via a  lower bound. Formally, we prove a minimax lower bound showing that there exists a problem instance such that \emph{any} learning algorithm suffers a regret of $\Omega(T^{2/3})$. Hence, our proposed algorithm is optimal. We construct a one dimensional problem with \emph{bumps} respecting the Lipschitz property. Even in this simple framework, we perturb the problem instance and using information theoretic framework, show that a regret lower bound of $\tilde{\Omega}(T^{2/3})$. This implies that \texttt{ZoomSIB-UCB} is optimal. To the best of our knowledge, this is the first work to obtain a minimax lower bound for general non-monotone reward functions.

\emph{Simulation:} We demonstrate the empirical robustness of \texttt{ZoomSIB-UCB} through extensive simulations on complex synthetic geometries and high-dimensional real-world datasets (Section \ref{sec:experiments}). Our results confirm that our approach strictly supersedes the previous state-of-the-art GSTOR \cite{kang2025single} in all respects, consistently achieving substantially lower cumulative regret across every evaluated environment.

\subsubsection{Novelty and Technical Challenges}
The algorithm with improved regret bound and the matching lower bounds comes with significant technical challenges, which we succinctly summarize here. Our proposed algorithm \texttt{ZoomSIB-UCB} uses the notion of zooming and discretization popular in Lipschitz bandits \cite{slivkins2019introduction,kleinberg2004nearly} along with the framework of sleeping bandits \cite{kanade-sleeping,kleinberg2010regret}. We first argue that a discretization leading to $N \sim T^{1/3}$ bins is enough in the first phase. A distinctive feature of \texttt{ZoomSIB-UCB} is that only a random subset of bins is available to the learner. Hence, a standard UCB treating all $N$ bins as viable options would be incorrect. Sleeping bandit \cite{kleinberg2010regret} accounts for this by comparing the current reward to the best available bin (see Algorithm~\ref{alg:zoomsib_ucb} for details). Note that a naive application of zooming based algorithm (which is typical in Lipschitz bandits) does not necessarily provide guarantees for single index model, as we need to estimate the underlying parameter $\theta_*$ along with learning the function simultaneously. We leverage a normalized Stein estimator, zooming based approach and sleeping bandits concurrently to complete the job.

For the minimax lower bound, construction of a hard problem instance and its perturbation is quite non-trivial. We consider a reward function by embedding small plateaued bump of height $\sim T^{-1/3}$. This generates a family of $2^{T^{1/3}}$ functions that are almost identical except  bumps (perturbation). We then reduce this to $T^{1/3}$ armed bandit problem. Finally we define 2 functions that differ in one bump and show that since the bump is of small height, distinguishing the two functions require huge exploration. Formally, using information theoretic tool like Bretagnolle-Huber inequality, we characterize this exploration, leading to an overall regret of $\tilde{\Omega}(T^{2/3})$. 

\subsection{Notation}
For integer $r$, $[r]$ denotes the set $\{1,\ldots,r\}$. The notation $\tildeO(\cdot)$ and $\tilde{\Omega}(\cdot)$ hides polylogarithmic terms. We use $\mathsf{polylog}(m)$ to denote polylogarithmic terms in $m$ with sufficiently large constant degree.

\section{Setup and Preliminaries}
We consider the sequential decision-making problem formulated as a Single-Index Bandit (SIB) \cite{kang2025single}. Let $T$ denote the time horizon. At each round $t \in [T]$, the learner observes an arm set $\mathcal{X}_t = \{x_{t,a} \in \mathbb{R}^d : a \in [K]\}$ consisting of $K$ feature vectors. We assume these context vectors are drawn i.i.d. from a continuous distribution $\mathcal{D}$ with a multivariate probability density function $p(\cdot)$.

The learner selects an arm $x_t \in \mathcal{X}_t$ and observes a stochastic reward:
$$y_t = f(x_t^\top \theta_*) + \eta_t$$
where $\theta_* \in \mathbb{R}^d$ is the unknown true parameter vector, and $f: \mathbb{R} \to \mathbb{R}$ is an unknown, continuously differentiable link function. 

\begin{assumption}[Noise]\label{asm:noise}
The noise sequence $\{\eta_t\}_{t \ge 1}$ satisfies: 
$\eta_t$ is independent of $x_t$ (and of all arms $\mathcal{X}_t$ 
at round $t$), and is conditionally $\sigma$-sub-Gaussian:
$\E[e^{\lambda \eta_t} \mid \mathcal{F}_{t-1}] 
\le e^{\lambda^2 \sigma^2 / 2},$ for all $\lambda \in \R,$ where $\mathcal{F}_{t-1} = \sigma(\mathcal{X}_{1:t}, y_{1:t-1})$ 
is the natural filtration.
\end{assumption}

\emph{Model Identifiability:} One needs to address two issues regarding identifiability in the single index model- (a) scaling and (b) direction (sign). To ensure scaling, we adopt the standard normalization $||\theta_*||_1 = 1$ (see \cite{kang2025single})\footnote{Another option is $\ell_2$ normalization. Our choice of $\ell_1$ normalization has no resemblance to sparsity. Our proposed algorithm uses a discretization which is more tailored to the $\ell_1$ norm.}. Moreover, a quantity of interest for single index model is $\mu^* = \mathbb{E}[f'(X^T\theta_*)]$ where $X\sim \mathcal{D}$. Similar to \cite{kang2025single} (see Section on general function), we assume $\mu^* > 0$ to resolve the direction identifiability.

\emph{Cumulative Regret:} Let $a_{t}^{*} := \arg\max_{a \in [K]} f(x^\top \theta_*$) denote the optimal arm at round $t$, and let $x_{t,*} := x_{t,a_t^*}$ The performance of the learner is measured by the cumulative regret over $T$ rounds, defined as the cumulative difference between the expected reward of the optimal arm and the chosen arm:
$$R_T = \sum_{t=1}^T \left( f(x_{t,*}^\top \theta_*) - f(x_t^\top \theta_*) \right)$$
Note that for non-monotone $f$, the optimal arm is not necessarily the one with the largest projection $x^\top\theta_* $

\subsection{Stein's Identity and Parameter Estimation}
Because the reward function $f(\cdot)$ is completely unknown, standard generalized linear bandit (GLB) estimators (such as maximum likelihood) are intractable. To bypass this structural limitation, we leverage Stein's method \cite{stein1981estimation}.

\begin{definition}[Score Function]
The score function $S^p: \mathbb{R}^d \to \mathbb{R}^d$ associated with the density $p(x)$ is defined as: $S^p(x) = -\nabla_x \log(p(x)) = -\frac{\nabla_x p(x)}{p(x)}$.
\end{definition}

For readability, we omit the superscript $p$ and simply write $S(x)$ when the underlying distribution is clear from the context.

\textbf{Stein's Identity:} By the Generalized Stein's Lemma, for any continuously differentiable function $f$, it holds that $\mathbb{E}[f(X) S(X)] = \mathbb{E}[\nabla f(X)]$ \cite{stein1981estimation}. Applying this to our specific reward model yields:
$$\mathbb{E}[y_t S(x_t)] = \mathbb{E}[f(x_t^\top \theta_*) S(x_t)] = \mathbb{E}[f'(x_t^\top \theta_*)] \theta_* := \mu_* \theta_*$$

This identity demonstrates that the expected value of the reward, when weighted by the score function, perfectly aligns with the direction of the true parameter $\theta_*$ (scaled by a constant $\mu_*$) \cite{stein1981estimation,kang2025single}. This allows the learner to construct an unbiased estimator for the direction of $\theta_*$ entirely without requiring the explicit form of $f(\cdot)$.

\begin{remark}[Regularity conditions for Stein's identity]
\label{rem:stein_regularity}
The identity 
$\E[f(X^\top\theta^*) S(X)] = \E[f'(X^\top\theta^*)] \theta^*$ 
follows from the Generalized Stein's Lemma 
\cite{stein1981estimation}, 
and requires the boundary term $f(x^\top\theta^*) p(x)$ to vanish 
as $\norm{x} \to \infty$.  Under Assumption~\ref{asm:lip}, the 
sub-Gaussian decay of $p$ and the local boundedness of $f$ on the 
effective range ensure this condition.  Additionally, the algorithm 
requires evaluation of the score function 
$S(x) = -\nabla \log p(x)$, and hence the density $p$ must be 
known to the learner.  This is consistent with the framework of 
\cite{kang2025single}, where the same requirement holds implicitly.
In practical applications such as recommendation systems, the 
feature distribution is typically estimated from logged data 
and can be treated as known.
\end{remark}

\subsection{The Truncated Stein Estimator}

Given $n$ independent and identically distributed samples $\{(x_i, y_i)\}_{i=1}^n$, the naive empirical approach would be to compute the sample average $\frac{1}{n}\sum_{i=1}^n y_i S(x_i)$. However, as shown in \cite{kang2025single}, a truncation is typically required to control the tails. Let $\varphi_\tau : \R^d \to \R^d$ denote an element-wise truncation function
\begin{align*}
    [\varphi_\tau(v)]_j = \text{sgn}(v_j) \min(\abs{v_j}, \tau), \quad \forall j \in [d]
\end{align*}
where $\tau > 0$ is a carefully chosen truncation threshold. The unnormalized truncated Stein estimator is then defined as
\begin{align}
\label{eqn:stein}
    \that = \frac{1}{n} \sum_{i=1}^n \varphi_\tau(y_i S(x_i))
\end{align}

\begin{remark}[Truncation threshold]
\label{rem:truncation_main}
The threshold $\tau$ balances the bias from discarding 
extreme values of $y_i S(x_i)$ against the variance 
reduction from truncation. We set
$\tau = \sqrt{\frac{3(\sigma^2 + L_f^2) M n}
{\log(2d/\delta)}},$
where $\sigma, L_f, M$ are as in 
Assumptions~\ref{asm:moment} and~\ref{asm:lip}. This 
yields the optimal $\mathcal{O}(d/\sqrt{n})$ estimation 
rate (see Lemma~\ref{lem:unnorm_stein}).
\end{remark}

To resolve identifiability, we define the normalized estimator as $\that_0 = \that/\norm{\that}_{1}$. We now establish the high-probability error bound for the normalized estimator $\that_0$.

\begin{lemma}[Normalized Stein Estimation Error] \label{thm:stein_error}
Suppose the unnormalized truncated Stein estimator satisfies $\norm{\that - \mu^* \tstar}_1 \le \epsilon$ with probability at least $1-\delta$, where $\epsilon = C_\theta d \sqrt{\frac{\log(2d/\delta)}{n}}$. Assume $\mu^* > 0$, $\norm{\tstar}_1 = 1$, and that the sample size $n$ is sufficiently large such that $\epsilon \le \mu^*/2$. Then, with probability at least $1-\delta$, the normalized estimator satisfies, $\norm{\that_0 - \tstar}_1 \le \frac{4\epsilon}{\mu^*}.$
\end{lemma}
Note that unlike \cite{kang2025single} the above gives a guarantee on the estimation of $\theta^*$, not a scaled version of it.

\begin{remark}[On Normalization and Norm Selection]
Lemma \ref{thm:stein_error} demonstrates that the $\ell_1$-normalization step does not degrade the fundamental convergence rate of the unnormalized estimator. \end{remark}

\section{Algorithm}
\label{sec:algo}
We discuss our proposed algorithm Zooming Single Index Bandit with UCB (\texttt{ZoomSIB-UCB})
\subsection{Overview}
\texttt{ZoomSIB-UCB} operates in two sequential phases. The high-level idea is to reduce the original high-dimensional contextual bandit problem to a sequence of one-dimensional bandit problems by exploiting the single-index structure. Since the reward depends on $x^\top \tstar$ alone, if we had access to $\tstar$ we could project every arm's context onto the real line and simply run a standard bandit algorithm on the resulting scalar values. Because $\tstar$ is unknown, we must first estimate it, and then carefully account for the estimation error when making decisions.

\textbf{Phase 1 (Parameter Estimation).} The algorithm dedicates the first $T_0$ rounds to estimating the unknown index direction $\tstar$. 

\emph{The Procedure:} During this phase, it pulls arms uniformly at random, collecting observations that are fed into the Stein-score estimator. The estimator produces $\that$ satisfying $\norm{\that - \mu^* \tstar}_1 \le \epsilon$ with high probability. We then work with the $\ell_1$-normalised version $\that_0 := \that / \norm{\that}_1$, which estimates $\tstar$ directly. 

\emph{The Calibration:} The exploration length $T_0$ is calibrated so that the resulting estimation error satisfies:
\begin{equation}
    \norm{\that_0 - \tstar}_1 \le \frac{\Delta}{2L},
\end{equation}
where $\Delta = T^{-1/3}$ is the bin width chosen for Phase 2 and $L = 4\sqrt{\log(dT/\delta)}$ is a high-probability upper bound on $\norm{x_{t,a}}_\infty$. 

This precise calibration is what makes the two phases mesh together: it guarantees that using $\that_0$ in place of $\tstar$ displaces each arm's projected index by at most $\Delta/2$, i.e., at most half a bin width.

\textbf{Phase 2 (UCB over Bins).} Once $\that_0$ is computed, the algorithm discretises the real line into $N = \calO(T^{1/3})$ bins $B_1, \ldots, B_N$ of width $\Delta$, centred symmetrically around zero within the window $[-W, W]$, where $W = 4\sqrt{\log(T/\delta)}$ captures all realistic index values with high probability. At each subsequent round $t$, the algorithm:
\begin{enumerate}
    \item Projects each of the $K$ presented arms onto the estimated index: $\zhat_{t,a} = x_{t,a}^\top \that_0$.
    \item Determines which bin each arm falls into, forming the available bin set $\mathcal{B}_t \subseteq [N]$.
    \item Runs UCB over the available bins---selecting the arm whose bin has the highest UCB index---and updates the sufficient statistics (pull count $n_j$ and reward sum $S_j$) for that bin.
\end{enumerate}

The core modelling assumption for Phase 2 is that all arms assigned to the same bin $B_j$ have nearly the same mean reward, approximated by $\mu_j := f(\tilde{z}_j)$ where $\tilde{z}_j$ is the bin centre. The within-bin approximation error is at most $L_{f'} \Delta$ by the Lipschitzness of $f$, and the estimation-induced displacement from Phase 1 adds a further $\Delta/2$ shift to the true index. Therefore, the total per-arm model error is bounded by $L_{f'} \Delta$. This slack is absorbed into an inflated UCB confidence radius, and the discretisation bias accumulated over $T$ rounds contributes $\calO(L_{f'} T^{2/3})$ to the regret.

\begin{algorithm}[t!]
\caption{Zooming Single-Index Bandit with UCB (\texttt{ZoomSIB-UCB})}
\label{alg:zoomsib_ucb}
\begin{algorithmic}[1]
\STATE \textbf{Input:} Total horizon $T$, number of arms $K$, 
failure probability $\delta$, sub-Gaussian 
parameter $\kappa$ of $\mathcal{D}$, noise parameter $\sigma$, 
reward bound $L_f$, score moment bound $M$.
\STATE \textbf{Initialize:} 
$L = 4\kappa\sqrt{\log(dTK/\delta)}$, 
$W = 4{\kappa}\sqrt{\log({TK}/\delta)}$, 
$\Delta = T^{-1/3}$, 
$N = \lceil 2W/\Delta \rceil$, 
$T_0 = \lceil d^2 \cdot T^{2/3} \cdot 
\mathsf{polylog}(dT/\delta) \rceil$, 
$\tau = \sqrt{3(\sigma^2 + L_f^2)\, M\, T_0 / \log(2d/\delta)}$.

\textbf{Phase 1: (Parameter Estimation)}
\FOR{round $t = 1, \dots, T_0$}
    \STATE Pull arm $a_t$ uniformly at random from 
    $\{1, \dots, K\}$ and observe reward $y_t$.
\ENDFOR
\STATE Compute unnormalized truncated Stein estimator $\that$ 
with threshold $\tau$ using 
$\{(x_{t,a_t}, y_t)\}_{t=1}^{T_0}$ (Eqn~\ref{eqn:stein}); 
set normalized index $\that_0 = \that / \|\that\|_1$.
\STATE Partition $[-W, W]$ into $N$ bins of width $\Delta$ with 
centres $\tilde{z}_j = -W + (j - 1/2)\Delta$. Init 
$n_j = 0, S_j = 0$ for all bins $j \in \{1, \dots, N\}$.

\textbf{Phase 2: (UCB over Bins)}
\FOR{round $t = T_0 + 1, \dots, T$}
    \STATE Observe current contexts $x_{t,a}$ for all 
    $a \in \{1, \dots, K\}$.
    \STATE For each arm $a$, compute 
    $\hat{z}_{t,a} = x_{t,a}^\top \that_0$, and assign bin 
    $b_{t,a} = \lceil (\hat{z}_{t,a} + W)/\Delta \rceil$ if 
    $|\hat{z}_{t,a}| \le W$; else $b_{t,a} = \perp$.
    \STATE Identify the set of currently available bins 
    $\mathcal{B}_t = \{ b_{t,a} : a \in \{1, \dots, K\}, 
    b_{t,a} \neq \perp$\footnote{\small $\perp$ denotes that arm $a$ is discarded at round $t$.}\}.
    \IF{$\mathcal{B}_t = \emptyset$}
        \STATE Pull arm $a_t$ uniformly at random from 
        $\{1, \dots, K\}$ and observe reward $y_t$.
    \ELSE
        \FOR{each available bin $j \in \mathcal{B}_t$}
            \STATE Compute the UCB index:
            \begin{equation*}
                \text{UCB}_j(t) = \begin{cases} 
                +\infty & \text{if } n_j = 0 \\ 
                \frac{S_j}{n_j} + \sigma
                \sqrt{\frac{2\log(2NT/\delta)}
                {n_j}} & \text{if } n_j \ge 1
                \end{cases}
            \end{equation*}
        \ENDFOR
        \STATE Pull arm 
        $a_t = \arg\max_{\{a : b_{t,a} \in \mathcal{B}_t\}} 
        \text{UCB}_{b_{t,a}}(t)$ and observe reward $y_t$.
        \STATE Update bin statistics: 
        $n_{b_{t,a_t}} = n_{b_{t,a_t}} + 1$ and 
        $S_{b_{t,a_t}} = S_{b_{t,a_t}} + y_t$.
    \ENDIF
\ENDFOR
\end{algorithmic}
\end{algorithm}

\textbf{\emph{Connection to zooming algorithms:}}
Standard zooming algorithms for Lipschitz bandits maintain an adaptive cover of the arm space: bins are created and refined on the fly, with finer resolution near high-reward regions. \texttt{ZoomSIB-UCB} instead uses a uniform partition of the projected index line at a fixed resolution $\Delta = T^{-1/3}$. The ``zooming'' here is therefore not adaptive in the classical sense---it is a one-shot reduction from a $d$-dimensional arm space to a uniform grid on $\R$, made possible by the single-index structure. The choice of $\Delta = T^{-1/3}$ directly resolves the fundamental bias-variance trade-off for discretizing Lipschitz functions. Grouping continuous arms into bins of width $\Delta$ introduces a per-round approximation bias of $\calO(\Delta)$, accumulating to $\calO(T\Delta)$ regret over the horizon. Conversely, making the bins finer increases the total number of bins $N \approx 1/\Delta$. The statistical cost to explore $N$ arms using UCB is $\calO(\sqrt{NT}) = \calO(\sqrt{T/\Delta})$. Minimising the total regret $\calO(T\Delta + \sqrt{T/\Delta})$ by balancing these two terms yields the optimal width $\Delta = T^{-1/3}$. At this resolution, the number of bins is $N = \calO(T^{1/3})$, ensuring both the discretization bias and the exploration regret scale at the optimal $\calO(T^{2/3})$ rate.

\textbf{\emph{Connection to sleeping experts and bandits:}} A distinctive and non-trivial feature of the algorithm is that at each round $t$ only a random subset $\mathcal{B}_t$ of bins is actually available: a bin $j$ is present in $\mathcal{B}_t$ only if at least one of the $K$ freshly drawn arms happens to fall inside $B_j$. Since the arms $x_{t,a}$ are drawn i.i.d. from a continuous distribution each round, $\mathcal{B}_t$ is random and changes with $t$. This is precisely the stochastic availability model of sleeping bandits \cite{kanade-sleeping,kleinberg2010regret}. A standard UCB analysis that treats all $N$ bins as perpetually available would be incorrect as a bin might go many rounds without any arm landing in it, and the UCB indices for absent bins should simply not be considered. The sleeping-UCB analysis correctly accounts for this by comparing the algorithm's reward only to the best available bin $j^*_t = \arg\max_{j \in \mathcal{B}_t} \mu_j$ at each round, rather than to the globally best bin. 

\begin{remark}[Sample splitting]\label{rem:splitting}
A crucial structural property of \texttt{ZoomSIB-UCB} is that 
$\hat\theta_0$ is computed entirely from Phase~1 data and is 
\emph{fixed} throughout Phase~2.  Therefore, the bin assignment 
$b_{t,a} = \lceil(\hat{z}_{t,a} + W)/\Delta\rceil$ is a 
deterministic function of $x_{t,a}$ and $\hat\theta_0$.  
Conditioned on Phase~1, the bin index of the selected arm is $\mathcal{F}_{t-1}$-measurable (it depends on the contexts $\mathcal{X}_t$ revealed at round $t$ before the reward $y_t$ is observed).  This makes the noise sequence within each bin a martingale difference sequence (MDS), which is essential for the regret analysis later.
\end{remark}

\section{Regret Analysis}
We bound the cumulative regret over the horizon $T$. Since the reward function $f$ is not assumed to be monotonically increasing, the arm with the largest projection does not necessarily yield the highest reward. Let $a^*_t = \arg\max_{a \in [K]} f(x_{t,a}^\top \tstar)$ denote the optimal arm at round $t$, and define its true index value as $z^*_t = x_{t,a^*_t}^\top \tstar$. The cumulative pseudo-regret is defined as: $ R_T = \sum_{t=1}^T \left[ f(z^*_t) - f(z_{t,a_t}) \right]$.
\subsection{Assumptions}
\begin{assumption}[Score Moment]
\label{asm:moment}
There exists a constant $M > 0$ such that $\mathbb{E}[S_j(X)^2] \le M$ for all $j \in [d]$, where $X \sim \mathcal{D}$.
\end{assumption}


\begin{remark}
Our Assumption \ref{asm:moment} requires only a bounded second moment of the 
score function, which is same as \cite{kang2025single} for 
the Stein-based estimator.

\end{remark}


\begin{assumption}[Context regularity and Lipschitz Continuity]
\label{asm:lip}
The context vectors are drawn i.i.d.\ from a continuous distribution 
$\mathcal{D}$ with a continuously differentiable density $p$ 
satisfying $p(x) > 0$ for all $x \in \mathbb{R}^d$, and each 
coordinate of $X \sim \mathcal{D}$ is $\kappa$-sub-Gaussian, i.e., 
$\mathbb{E}[e^{\lambda X_j}] \le e^{\kappa^2 \lambda^2 / 2}$ for 
all $\lambda \in \mathbb{R}$ and $j \in [d]$.   Furthermore, the unknown reward 
function $f : \mathbb{R} \to \mathbb{R}$ is continuously 
differentiable, and there exist constants $L_f, L_{f'} > 0$ such 
that $|f(z)| \le L_f$ and $|f'(z)| \le L_{f'}$ for all 
$|z| \le C_W\sqrt{\log(dTK/\delta)}$, where $C_W$ is a universal constant. 
\end{assumption}

\begin{remark}
The sub-Gaussian coordinate condition is standard in contextual 
bandits \cite{yadkori,dani2008stochastic} and is adopted by 
\citep[Assumption~2.3 and Appendix~H.2]{kang2025single}.  Gaussian, 
bounded, and sub-Exponential distributions all satisfy it.  The 
boundedness of $f$ is a local condition on the effective range 
$|z| \le L = \tilde{O}(1)$: since we showed the projected indices satisfy 
$|x_{t,a}^\top \theta^*| \le W/2$ with high probability (Lemma \ref{lem:E3}), $f$ need not be globally bounded.  This 
accommodates unbounded reward functions used in our experiments.
\end{remark}



\begin{theorem}[Main Regret Bound] \label{thm:main_regret}
Suppose Assumptions~ \ref{asm:noise}, \ref{asm:moment} and \ref{asm:lip} hold, with truncation threshold 
$\tau = \sqrt{3(\sigma^2 + L_f^2)MT_0/\log(2d/\delta)}$ and
that $\mu^* \geq \frac{c_0}{\mathsf{polylog}(dT/\delta)}$ 
where $\delta \in (0,1)$, $c_0$ is a sufficiently small positive constant. Then, with probability at least $1 - \delta$,
\begin{equation*}
    R_T \le \calO\left( d^2 T^{2/3} 
    {\mathrm{polylog}}\left(\frac{dT}{\delta}
    \right) \right) + 4\sigma\sqrt{NT
    \log\left(\frac{2NT}{\delta}\right)} 
    + 4L_{f'} T^{2/3} = \tildeO(d^2 T^{2/3})
\end{equation*}
\end{theorem}
Let us discuss the consequences of the above result.

\begin{remark}[Improved Regret over \cite{kang2025single}]
    Note that our result is a direct improvement over \cite{kang2025single}, where the authors prove that for general (non-monotone) functions the regret is $\tildeO(T^{3/4})$. On the other hand, we get an improved regret.
\end{remark}
\begin{remark}[Dependence on $d$]
    We obtain a worse dependence on $d$ compared to \cite{kang2025single}. However, this gets amortized by improved dependence on $T$ as seen in experiments (see Figure~\ref{fig:exp2}, Table~\ref{tab:exp2}) even for moderately large $T$.
\end{remark}

\begin{corollary}[Expected regret]\label{cor:expected_regret}
Under the same conditions as Theorem~\ref{thm:main_regret}, 
setting $\delta = 1/T$, $\mathbb{E}[R_T] = \tilde{O}(d^2 T^{2/3}).$
\end{corollary}

\begin{proposition}[Modular reduction]\label{prop:modular}
Let $\mathcal{A}$ be any bandit algorithm for $N$ arms with 
stochastic availability, achieving regret 
$\mathrm{Reg}_{\mathcal{A}}(N,T)$ against the best available arm.  
Run Phase~1 to obtain $\hat\theta_0$, then run $\mathcal{A}$ on 
the $N$ bins.  The total regret satisfies:
\[
R_T \le R_T^{(1)} + \mathrm{Reg}_{\mathcal{A}}(N,T) + 2L_{f'}T\Delta.
\]
\end{proposition}

\begin{corollary}[Instance-dependent bound]\label{cor:instance}
Instantiating with UCB, let $\Delta_j = \mu_{j^*} - \mu_j$ be the 
gap of bin $j$ and $s = L_{f'}\Delta$.  Then:
\[
R_T^{(2)} \le O\!\left(\sum_{j:\,\Delta_j > 2s} 
\frac{\sigma^2\log(NT/\delta)}{\Delta_j - 2s}\right) + 2sT.
\]
For instance, when the reward function $f$ has a unique well-separated maximum so that only $O(1)$ bins have $\Delta_j \le 2s$, the first term is $O(N\log T) = 
O(T^{1/3}\log T)$, improving over the worst-case $T^{2/3}$.
\end{corollary}

\paragraph{Condition on $\mu^*$ and choice of $T_0$:}
This is a mild quantitative strengthening of the condition $\mu^* > 0$. It rules out link functions whose $\mu^*$ is arbitrarily close to $0$. Note that if $\mu^*$ is very small, then the stein estimator gives (nearly) vacuous results and such an assumption makes it well-behaved. Note that our algorithm does not require the knowledge of $\mu^*$. Moreover, in experiments (Section~\ref{sec:experiments}) we use an ``adaptive stopping'' rule for exploration time $T_0$ which completely removes the issue of choosing $T_0$ in practice.


Notably, in \cite{kang2025single} (Section  3.5 for non-monotone functions), the authors implicitly assume the knowledge of $\mu^*$, to choose exploration duration for their kernel based estimator and absorb it in the $\mathcal{O}(.)$ notation making their algorithm non-implementable without knowing $\mu^*$. On the other hand, we do not require the knowledge of $\mu^*$. As long as it is bounded away from $0$, our proposed algorithm works. Is we assume the knowledge of $\mu^*$ similar to \cite{kang2025single}, we can remove the condition on $\mu^*$ as stated in Theorem~\ref{thm:main_regret} completely. However we believe that the knowledge of $\mu^*$ is a rather strong assumption, and hence make an attempt to remove it.

\subsection{Proof Sketch}
We first define a good event $\mathcal{E}$.
On $\mathcal{E}$, we show that the estimated and true projected indices are uniformly close:
\begin{equation} \label{eq:projection_close}
    |\zhat_{t,a} - z_{t,a}| = |x_{t,a}^\top(\that_0 - \tstar)| \le \norm{x_{t,a}}_\infty \norm{\that_0 - \tstar}_1 \le L \left( \frac{\Delta}{2L} \right) = \frac{\Delta}{2}
\end{equation}
Assuming $T$ is sufficiently large such that $\Delta \le W$, on $\mathcal{E}$ every true index satisfies $|z_{t,a}| \le W/2$. Combining this with \eqref{eq:projection_close} yields $|\zhat_{t,a}| \le W/2 + \Delta/2 \le W$. This means every arm is successfully assigned to some bin. In particular, $\mathcal{B}_t \neq \emptyset$ for every $t > T_0$, ensuring the algorithm never falls back to the random-pull branch during Phase 2.

Conditioned on $\mathcal{E}$, the regret is cleanly decomposed $R_T = R_T^{(1)} + R_T^{(2)}$ into Phase 1 and Phase 2. Lemma \ref{lem:phase1} trivially bounds the uniform exploration cost of Phase 1. For Phase 2, we decompose the instantaneous regret into two parts- (a) discretization bias and (b) (core) bandit regret.

The discretization bias is handled by the Lipschitzness of $f$ (Lemma \ref{lem:disc_bias}). In particular, this yields a regret of $\mathcal{O}(L_{f'} T^{2/3})$. The core bandit regret is handled by a novel sleeping-UCB analysis over the stochastic active bin sets (Lemma \ref{lem:sleeping_ucb}). For $N$ bins, the regret scales with $\sqrt{NT}$. We choose $N=\mathcal{O}(T^{1/3})$ (ignoring log factor), which gives the regret from this stage as $\tildeO(T^{2/3})$. Finally, combining these two gives the desired bound.

\section{Minimax Lower Bound}
We establish the matching information-theoretic lower bound, demonstrating that no algorithm can achieve a regret scaling of $\Omega(T^{2/3})$ in the worst case for the general single-index bandit. 

\begin{theorem}[Minimax Lower Bound] \label{thm:lower_bound}
Let the time horizon $T$ be sufficiently large. Under the standard regularity conditions (Assumptions~\ref{asm:moment},\ref{asm:lip}), there exist an absolute constant $c > 0$ and $T_0 \in \mathbb{N}$ such that for all $T \ge T_0$ and every policy $\pi$,
\begin{equation*}
    \sup_{(\tstar, f, \mathcal{D})} \E\left[R_T(\pi)\right] \ge c T^{2/3},
\end{equation*}

where the supremum is taken over valid problem instances with ambient dimension $d=1$, parameter $\tstar = +1$, context distribution $\mathcal{D} = N(0,1)$, bins on $[0,1]$ and the number of features $K = \lceil4N \log T/p_{\min}\rceil$ with $p_{\min} = \phi(1)$ and $\mu^* \ge c1/(2 \log T)$.
\end{theorem}

\begin{remark}
    Combined with the $\tildeO(T^{2/3})$ upper bound established in Theorem \ref{thm:main_regret}, this confirms that the minimax regret for the non-monotone ,1Single-Index Bandit is strictly $\tilde{\Theta}(T^{2/3})$. 
\end{remark}

\subsection{Proof Sketch}
The lower bound relies on a three-step reduction strategy: construct a family of nearly-identical problem instances, show that any algorithm must be confused between them for a long time, and conclude via information theory that this confusion forces $\Omega(T^{2/3})$ regret on at least one instance.

\textbf{Step 1: Constructing a Localized Hard Family} We work in dimension $d=1$ with $\theta^*=1$ and standard Gaussian contexts $\mathcal{D}=N(0,1)$. We partition the interval $[0,1]$ into $N = \lceil T^{1/3} \rceil$ equally spaced bins. For each Boolean vector $\beta \in \{-1,+1\}^N$, we define a link function $f_\beta(z)$ $=\tfrac{1}{2} + \lambda\,\rho(z) + \varepsilon\sum_{j=1}^N \beta_j\,\psi_j(z)$ by embedding small, plateaued bump functions $\psi_j$ of height $\varepsilon = 1/(2N)$ within each bin on top of a common smooth \emph{ramp} $\rho$ scaled by $\lambda = 1/\log T$ (Lemma \ref{lem:lb_validity}). Here $\rho$ is a fixed, infinitely differentiable function that rises from $0$ to $1$ over $(-1,0)$ and is flat ($\rho\equiv 1$) on $[0,1]$; it therefore never interferes with the bumps, yet its slope on $(-1,0)$ guarantees a strictly positive expected Stein score $\mu^* \ge c_1\lambda/2 \ge c_0/\mathrm{polylog}(T) > 0$ (Lemma \ref{lem:lb_validity}(d)), ensuring every instance is a \emph{valid} single-index model with non-trivial signal rather than a degenerate one. This generates a family of $2^N$ functions that are globally identical except for $\pm \varepsilon$ perturbations localized inside the bins, strictly satisfying the boundedness and Lipschitz assumptions.

\textbf{Step 2: Reducing SIB to $N$-armed bandit problem (Section \ref{sec:lb_availability}).} The primary technical hurdle in the single-index bandit is the ``sleeping" structure: at each round, only a random subset of bins contains arms. We remove this complication away by provide a massive menu of arms ($K = \lceil 4N \log T / p_{\min} \rceil$) so that, with high probability, every single localized inner bin contains at least one available arm at every round (Lemma \ref{lem:lb_availability}). Conditioned on this universal availability event $\mathcal{E}$, the sleeping structure vanishes, and the problem collapses into a pure, perpetually available $N$-armed stochastic bandit in which every pull of a downward-pointing bin incurs instantaneous regret at least $\varepsilon$ (relative to the optimal reward $\tfrac{1}{2}+\lambda+\varepsilon$ attained on an active up-bump bin) (Lemma \ref{lem:lb_per_round_regret}).

\textbf{Step 3: The Information-Theoretic Trap (Section \ref{sec:lb_kl_bh}).} To trap the algorithm, we isolate a ``null" reference hypothesis $\beta^{(0)}$ (where every bump points downward, making every bin slightly sub-optimal) and $N$ parallel ``spike" hypotheses $\beta^{(j)}$ (identical to the null, except bin $j$ is flipped to point upward, making it uniquely optimal). Since the ramp term $\lambda\rho$ is shared by all hypotheses and is flat over the bins, it cancels in every pairwise comparison, so null and spike differ only by the $\pm\varepsilon$ bump in bin $j$ (a mean gap of $2\varepsilon$). Because the bump height $\varepsilon$ is extremely small, distinguishing the null from a spike requires massive exploration (Proposition \ref{prop:lb_kl}).

The Bretagnolle-Huber inequality formalizes this into an inescapable dilemma for any algorithm: if the algorithm plays conservatively and explores bin $j$ rarely, it cannot identify the optimal arm when placed in the spike hypothesis; conversely, if it explores aggressively to check for a spike, it suffers heavy regret pulling a downward-pointing bump if the true environment is the null hypothesis. Balancing this bias-variance tradeoff strictly dictates the choice $N = \Theta(T^{1/3})$, forcing a minimum expected regret of $\Omega(T^{2/3})$.
\section{Experiments}
\label{sec:experiments}

\begin{figure*}[htbp]
    \centering
    \begin{subfigure}[b]{0.245\textwidth}
        \centering
        \includegraphics[width=\linewidth, height = .135\textheight, page=1]{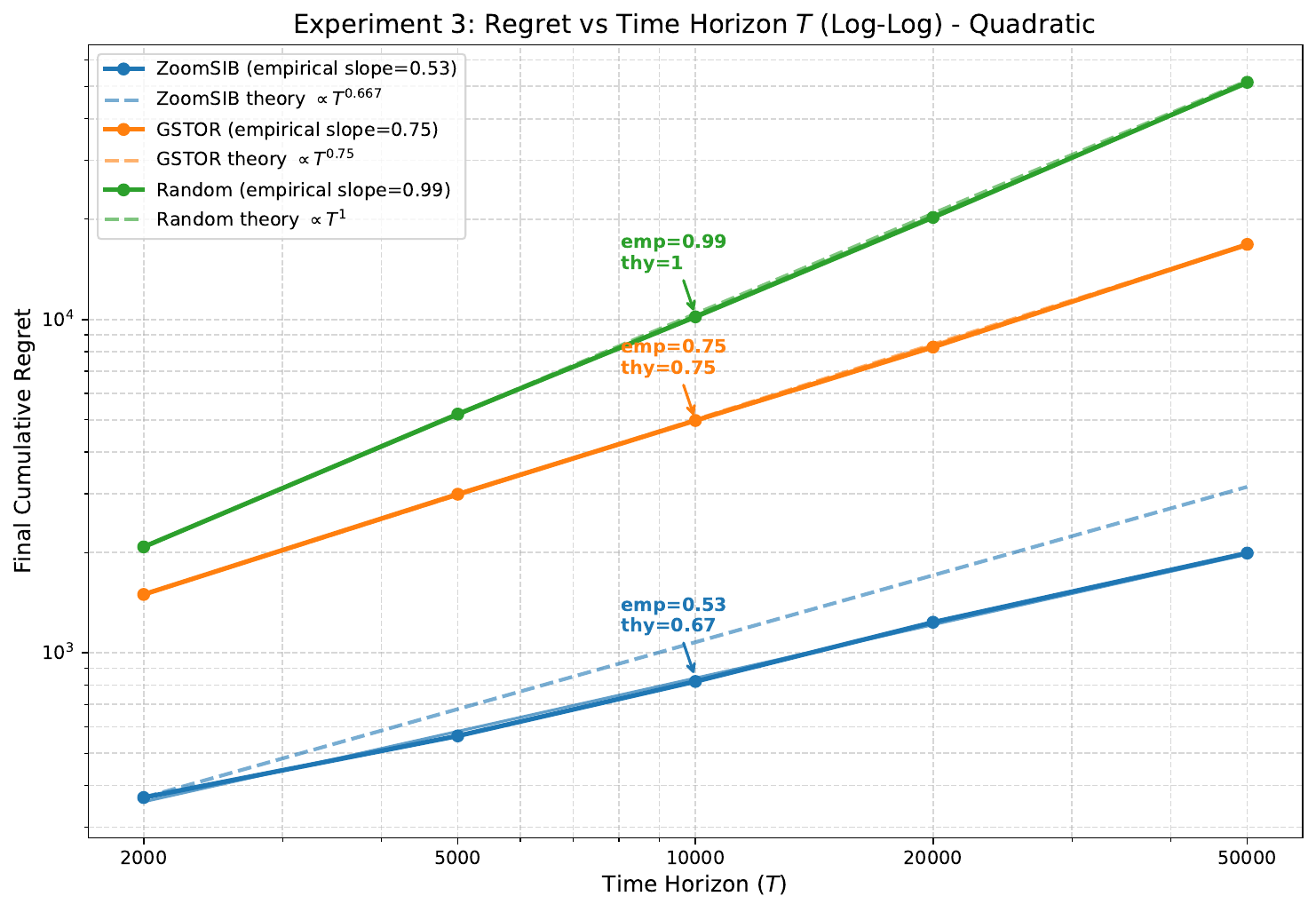}
        \caption{Log-Log (Quadratic)}
    \end{subfigure}
    \hfill
    \begin{subfigure}[b]{0.245\textwidth}
        \centering        \includegraphics[width=\linewidth,height = .135\textheight, page=2]{images/adaptive_time_scaling.pdf}
        \caption{Horizon (Asymmetric)}
    \end{subfigure}
    \hfill
    \begin{subfigure}[b]{0.245\textwidth}
        \centering
        \includegraphics[width=\linewidth, height =.135\textheight, page=3]{images/adaptive_time_scaling.pdf}
        \caption{Horizon (Zigzag)}
    \end{subfigure}
    \hfill
    \begin{subfigure}[b]{0.245\textwidth}
        \centering
        \includegraphics[width=\linewidth,height = .135\textheight, page=1]{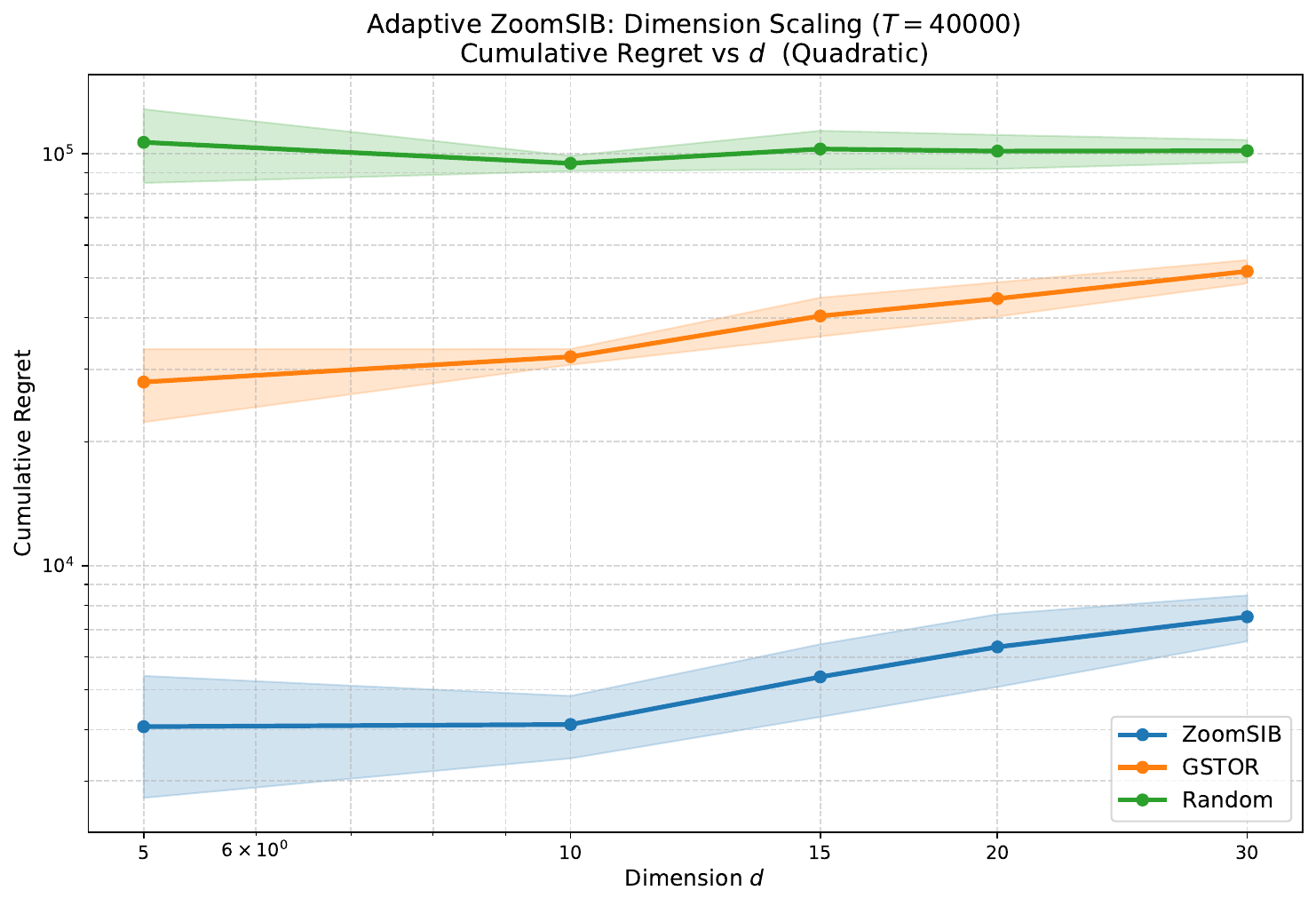}
        \caption{Dimension (Quadratic)}
    \end{subfigure}
    \vspace{-0.5em}
    \caption{\textbf{Comprehensive Synthetic Evaluation.} \textbf{(a)} Log-log regression of cumulative regret. The empirical slope of $0.53$ falls strictly below the dashed theoretical limit, confirming robust sublinear scaling. \textbf{(b, c)} Cumulative regret over varying time horizons. \texttt{ZoomSIB-UCB} leverages its early adaptive stopping rule to immediately pull ahead of the GSTOR baseline. \textbf{(d)} Dimension scaling at $T=40{,}000$. By operating entirely in the projected 1D index space after the initial estimation phase, ZoomSIB-UCB remains highly resilient to massive increases in feature dimension ($d$).}
    \label{fig:main_synthetic}
\end{figure*}
We evaluate \texttt{ZoomSIB-UCB} against three baselines: GSTOR (non-monotone, kernel-based),
ESTOR (monotone-assumed), and Random (uniform selection). We stress-test shape-agnosticism
using three synthetic link functions: Quadratic $f(z)=-(z-1)^2+1$, Asymmetric
$f(z)=z e^{-z^2}$, and Zigzag $f(z)=\sin(z)+0.3z$. Default settings: $d=10$,
standard Gaussian contexts, $K=20$, averaged over $30$ trials ($\pm 1$ std.\ dev.).

\paragraph{Adaptive Stopping: No Knowledge of $T_0$ (or $\mu^*$) Required.}
Although theory prescribes $T_0 = \tilde{\mathcal{O}}(d^{2}T^{2/3})$, which can be
prohibitively large, we equip \texttt{ZoomSIB-UCB} with an adaptive stopping rule that
monitors parameter-estimate stability in real time, autonomously transitioning to
exploitation once estimates converge. As shown in Table~\ref{tab:exp1}, this
self-tuning mechanism exits exploration at just $1.8\%$ of the horizon for
easy functions with strong gradients (e.g., Quadratic), while waiting longer
for flatter geometries (e.g., Asymmetric).


\paragraph{Scaling and Theoretical Verification.} 
\texttt{ZoomSIB-UCB} achieves strictly sublinear regret across all horizons and dimensions, consistently outperforming both baselines (Figure~\ref{fig:main_synthetic}b--d). The gap is most pronounced on Quadratic and Zigzag functions, where GSTOR incurs up to $10\times$ higher regret at $T=40{,}000$ (Table~\ref{tab:exp1}). By compressing $d$-dimensional features into a $1$D index before arm selection, \texttt{ZoomSIB-UCB} sidesteps the curse of dimensionality: a $6\times$ increase in ambient dimension inflates regret by only $2.3\times$ (Table~\ref{tab:exp2}, Figure~\ref{fig:main_synthetic}d).

To verify asymptotic guarantees, we plot regret on a log-log scale (Figure~\ref{fig:main_synthetic}a), where slope encodes the scaling exponent. \texttt{ZoomSIB-UCB} achieves an empirical slope of $0.53$, well below the theoretical ceiling of $2/3\ (\approx 0.667)$, confirming optimal scaling and that the worst-case bound is tight only under adversarial conditions. See also Figures~\ref{fig:exp3_asym},~\ref{fig:exp3_zig}.

\begin{figure}[t!]
    \centering
    \begin{subfigure}[b]{0.3\textwidth}
        \centering
        \includegraphics[width=\textwidth]{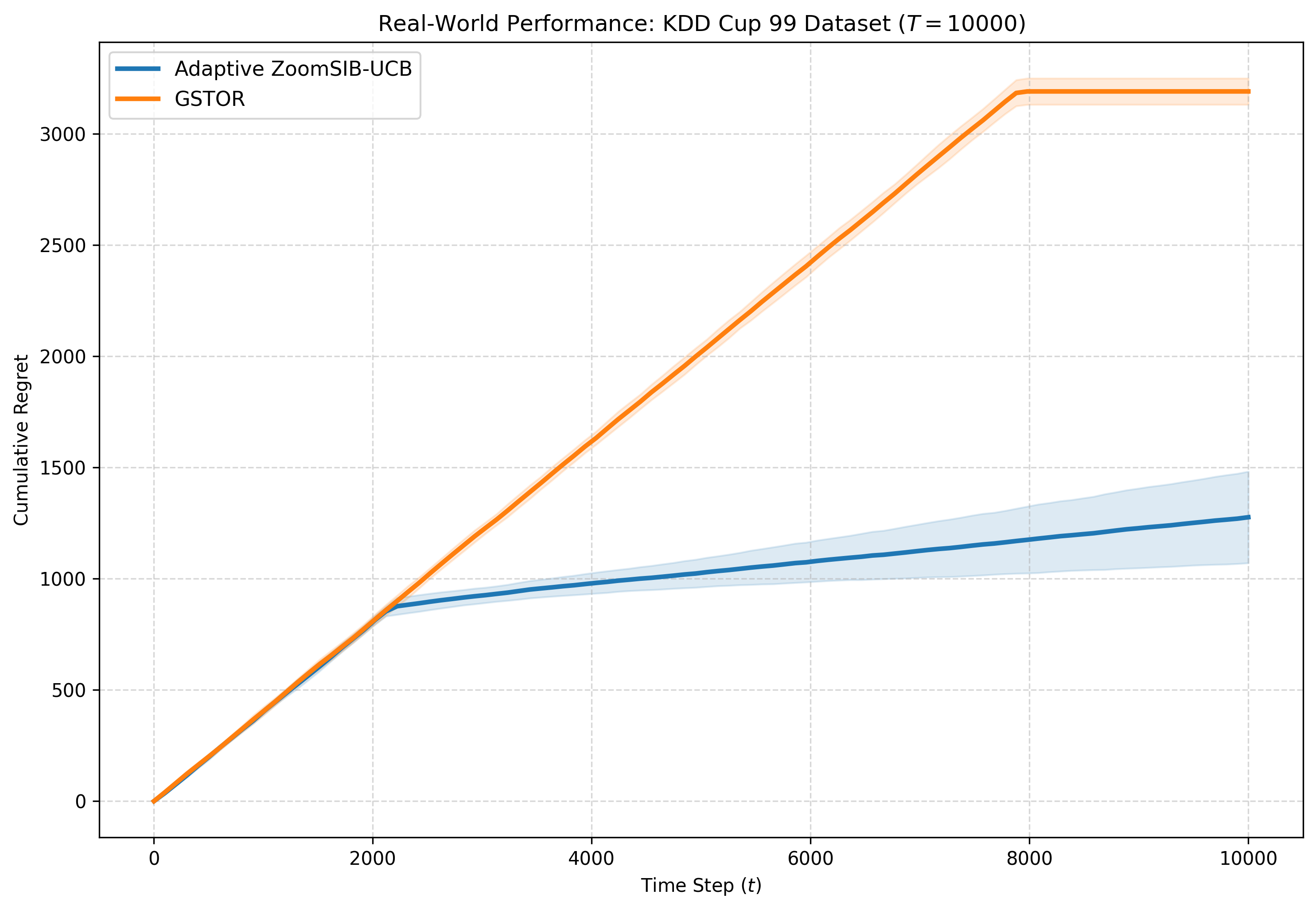}
        \caption{KDD Cup 99 ($d=39$)}
    \end{subfigure}
    \hspace{0.1\textwidth}
    \begin{subfigure}[b]{0.3\textwidth}
        \centering
        \includegraphics[width=\textwidth]{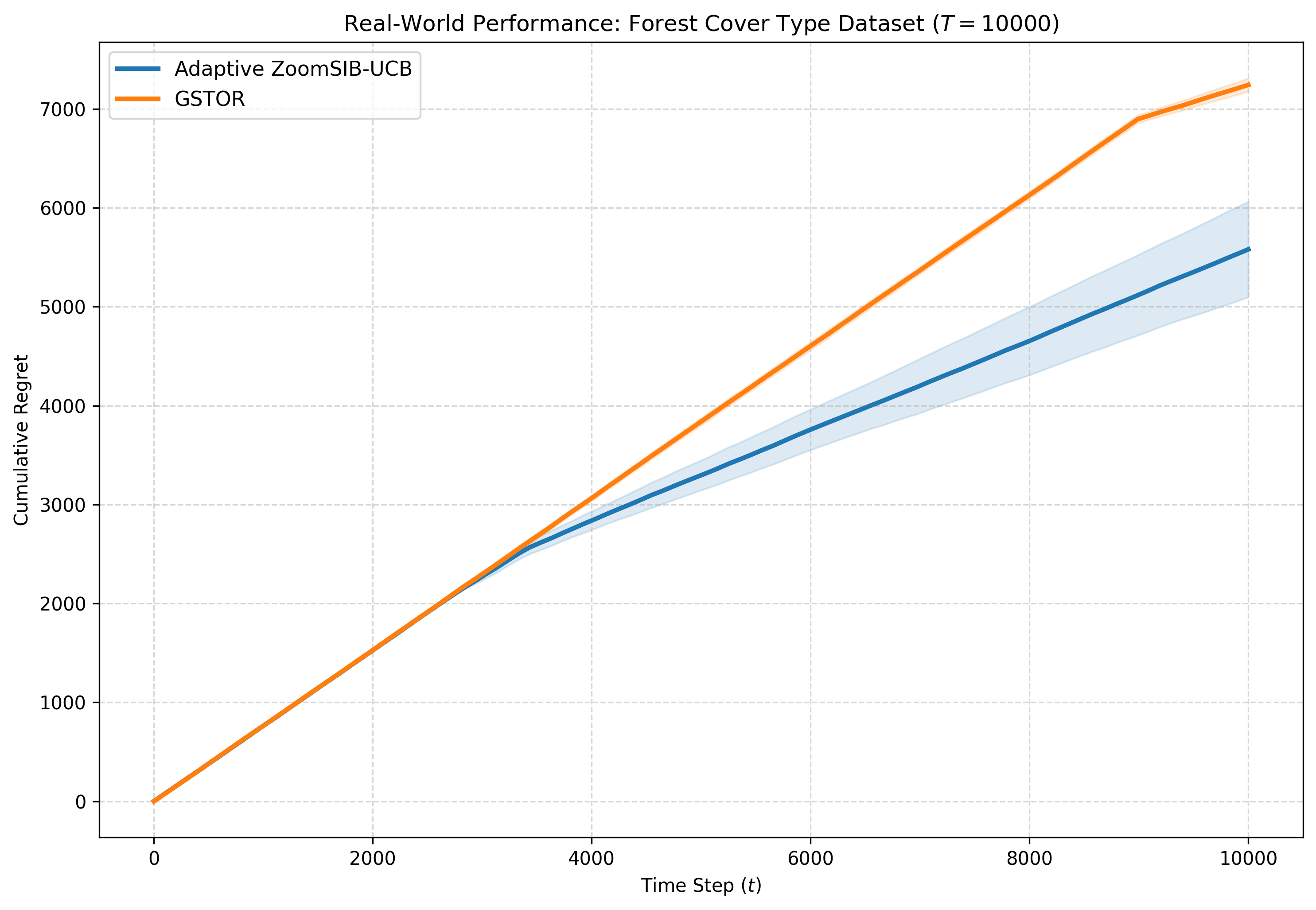}
        \caption{Forest Cover Type ($d=55$)}
    \end{subfigure}
    \caption{\textbf{Real-World Datasets.} Cumulative regret on datasets converted to multi-armed bandit tasks ($T=10{,}000$, $K=32$). Adaptive ZoomSIB-UCB successfully detects early signal convergence, bypassing the catastrophic $79\%$-$89\%$ exploration phases forced upon GSTOR.}
    \label{fig:real_world_appendix}
\end{figure}

\paragraph{Real-World Data.} We evaluated \texttt{ZoomSIB-UCB} on two offline classification datasets—KDD Cup 99 and Forest Cover Type—cast as contextual bandit tasks (Figure~\ref{fig:real_world_appendix}). GSTOR's rigid constraints force exploration over up to $89\%$ of the horizon, whereas \texttt{ZoomSIB-UCB}'s adaptive mechanism identifies feature patterns early, limiting exploration to just $21\%$--$29\%$ and substantially reducing cumulative regret.

Full experimental setups, tabular data, and misspecification evaluations (where ESTOR catastrophically collapses to linear regret) are detailed in Appendix~\ref{app:experiments}.

\section{Conclusion and Open Problems}
This paper resolves an open problem by establishing $\tilde{\Theta}(T^{2/3})$ as the minimax-optimal regret rate for the non-monotone single-index bandit, improving upon $\tilde{O}(T^{3/4})$ with the first matching information-theoretic lower bound in this setting. We end the paper with a few open problems. Our analysis relies on the Stein estimator and requires $\mu^* \neq 0$; can an alternative technique be developed? It would also be interesting to explore single-index bandits beyond bounded Lipschitz structure, and to understand the implications of our result in the multi-index setting. We keep these as our future endeavors.

\bibliography{ref}
\bibliographystyle{plainnat}

\appendix

\section{Other Related Work}
\subsection{Linear Bandits, Contextual Bandits and Generalized Linear Bandits}
Linear bandits extend the finite-armed bandit problem to structured continuous 
action spaces, with \cite{auer2002using} establishing the foundational 
optimism-under-uncertainty template and \cite{dani2008stochastic} providing 
the first $\tilde{O}(d\sqrt{T})$ regret bound via confidence ellipsoids, later 
sharpened to a tight $\tilde{O}(d\sqrt{T})$ by the canonical OFUL algorithm of 
\cite{yadkori} using regularized least-squares estimation. The 
matching $\Omega(d\sqrt{T})$ lower bound of \cite{dani2008stochastic} confirms 
optimality. The adversarial setting is studied by \cite{abernethy2008competing} 
and \cite{bubeck2012towards}, who establish near-optimal $\tilde{O}(\sqrt{dT})$ 
bounds, while subsequent work extended the framework to sparsity 
\citep{carpentier2012bandit}, low-rank structure \citep{jun2019bilinear}, and 
misspecification \citep{lattimore2020learning}. Contextual bandits condition 
rewards on an observed context vector, with LinUCB \citep{li2010contextual} and 
its theoretical analysis \citep{chu-contextual} achieving $\tilde{O}(d\sqrt{T})$ 
regret in the linear case; the agnostic setting is addressed by 
\cite{langford2008epoch} via epoch-greedy exploration and by 
\cite{dudik2011efficient} via oracle reductions, with \cite{foster2020beyond} 
ultimately achieving the optimal $\tilde{O}(\sqrt{KT})$ rate through regression 
oracle reductions. 

Generalized linear bandits (GLBs) pass the inner product 
$\langle \theta, a \rangle$ through a known link function $\mu(\cdot)$, with 
GLM-UCB \cite{filippi-glm} establishing $\tilde{O}(d\sqrt{T})$ regret 
under bounded link functions, \cite{li2017provably} improving curvature 
dependence and \cite{faury2020improved} achieving refined bounds for logistic 
bandits via self-concordance. Also \cite{abeille2021instance} derive 
instance-dependent lower bounds that connect GLBs to the broader theory of 
structured bandits.

\subsection{Single Index Model}
Semiparametric efficiency bounds and optimal estimation rates for the SIM 
in the classical low-dimensional regime are established by 
\cite{ichimura1993semiparametric} and \cite{hardle1993optimal}. \cite{hornik1989multilayer} and \cite{mccullagh1989generalized} connect the 
SIM to neural networks and generalized linear models respectively, situating it 
within a broader modeling landscape. In the high-dimensional setting where 
$d \gg n$, \cite{plan2016generalized} provides a seminal analysis showing that 
the index vector $\theta^*$ can be estimated from one-bit compressed measurements 
at the rate $O(s \log d / n)$ under sparsity, leveraging the connection between 
the SIM and generalized linear measurements. \cite{brillinger1982generalized} 
has earlier observed that ordinary least squares recovers $\theta^*$ up to a 
scalar in the SIM under Gaussian designs, a phenomenon later formalized and 
extended by \cite{plan2016generalized} and \cite{thrampoulidis2015lasso}. 
The work of \cite{kalai2009isotron} introduce the Isotron algorithm, the first 
computationally efficient method to jointly learn $\theta^*$ and $f$ with 
provable guarantees under Gaussian covariates. \cite{kakade2011efficient} 
subsequently extend these results via the STAP algorithm, achieving polynomial 
sample and computational complexity for learning SIMs with monotone link 
functions.

Through the lens of empirical process theory and Gaussian width, the estimation 
error of convex programs for SIM recovery was characterized in terms of the 
geometry of the parameter set by \cite{plan2016generalized} and 
\cite{goldstein2018structured}. More recently, the 
connection between SIMs and gradient descent on two-layer neural networks has 
attracted significant attention: \cite{barak2022hidden} and \cite{abbe2022merged} 
studied the SIM as a prototypical model for understanding what gradient-based 
methods can and cannot learn, establishing computational-statistical gaps tied 
to the information exponent of the link function $f$. \cite{bietti2022learning} 
and \cite{damian2022neural} further showed that a single gradient step on the 
population loss suffices to recover the index direction up to a polynomial sample 
complexity, formalizing the so-called leap exponent as the key quantity governing 
the difficulty of the learning problem under staircase-structured link functions.

Closest to our setting among recent work is that of \citet{arya2025batched},
who study a \emph{batched} global bandit with covariates under a single-index
structure and propose the \textsc{BIDS} algorithm, combining sliced inverse
regression with dynamic binning along the estimated index. Their formulation
differs from ours on four counts. They operate in the batched setting with $M$
pre-specified batch endpoints, at which alone the policy may be updated, whereas
we are fully sequential. They observe a single covariate $X_t$ per round and
select among a fixed arm set with \emph{arm-specific} links
$g^{(k)}(x) = f^{(k)}(x^\top\beta_0)$, whereas we observe $K$ i.i.d.\
arm-feature vectors under a \emph{common} link $f$; indeed
\citet{arya2025batched} themselves note that the formulation of
\citet{kang2025single}, which we adopt, ``differs substantially'' from theirs.
Their rates are driven by a margin parameter $\alpha$, giving $T^{\beta_M}$ with
$\beta_M = \frac{1-\gamma}{1-\gamma^M}$ and $\gamma = \frac{1+\alpha}{3}$, and
monotonicity of the link plays no role in their analysis; we impose no margin
condition and study precisely the non-monotone case. Finally, their Assumption~3
requires bounded context support with the projected density bounded above and
below, and their estimation rate is \emph{assumed} as a black box (their
Assumption~6), while we require only coordinatewise sub-Gaussian contexts and
prove the corresponding rate for the truncated Stein estimator. Their optimality
claim is also conditional, holding only for fixed $K$, known $\alpha$, and a
supplied pilot direction; without a pilot they obtain
$\widetilde{\mathcal{O}}(\max\{K^{1/3}T^{2/3}, K^{1-\beta_M}T^{\beta_M}\})$. We
note that the two sets of results agree where they should: in the no-margin
limit $\alpha \to 0$ with $M \to \infty$ one has $\beta_M \to 2/3$, matching the
$\widetilde{\Theta}(T^{2/3})$ rate established here.

\subsection{Lipschitz Bandits}
In \cite{agrawal1995continuum}, the authors  establish the feasibility of 
consistent estimation in the one-dimensional setting, and \cite{kleinberg2004nearly} 
derived nearly tight $\tilde{O}(T^{2/3})$ regret bounds, later tightened and 
connected to nonparametric regression by \cite{auer2007improved}. 
\cite{kleinberg2008multi} extend the framework to general metric spaces via 
the zooming algorithm, achieving instance-dependent regret scaling with the 
zooming dimension rather than the ambient dimension, with a matching minimax 
lower bound of $\Omega(T^{(d+1)/(d+2)})$ for $d$-dimensional spaces confirming 
the tightness of this dependence. The $\mathcal{X}$-armed bandit framework of 
\cite{bubeck2011x} introduce the HOO algorithm, which operates on a 
hierarchical partition of the action space and achieves regret depending on the 
near-optimality dimension, providing a unified treatment of Lipschitz and 
smooth bandits; \cite{munos2011optimistic} subsequently develop the 
parameter-free SOO and StoSOO algorithms, and \cite{valko2013stochastic} 
extend this line of work to kernel-based methods exploiting RKHS structure. 
\cite{slivkins2014contextual} generalize the framework to contextual bandits 
via a joint similarity metric over contexts and actions, \cite{bubeck2011lipschitz} 
derive instance-dependent logarithmic regret bounds under a margin condition 
analogous to Tsybakov's noise condition, and \cite{locatelli2018adaptivity} 
study adaptive algorithms achieving optimal rates across smoothness classes 
without knowledge of the Lipschitz constant.

\section{Proofs}
\label{sec:proofs}

To establish the convergence of our normalized estimator, we first require a high-probability bound on the unnormalized truncated Stein estimator. Under our specific moment and regularity assumptions, this bound is achieved by carefully setting the truncation threshold $\tau$. For completeness, we adapt the underlying concentration argument  to our setting:
\begin{lemma}[Unnormalized Stein Estimation Error; 
cf.\ {\cite[Theorem~3.1]{kang2025single}}]
\label{lem:unnorm_stein}  Let $\{(x_i, y_i)\}_{i=1}^n$ be 
i.i.d.\ samples with $y_i = f(x_i^\top\theta^*) + \eta_i$, 
$x_i \sim \mathcal{D}$, $\eta_i$ independent of $x_i$ with 
$\mathbb{E}[\eta_i] = 0$ and 
$\mathbb{E}[\eta_i^2] \le \sigma^2$.  Under 
Assumptions~\ref{asm:moment} and~\ref{asm:lip}, with 
truncation threshold
\[
\tau = \sqrt{\frac{3(\sigma^2 + L_f^2)\, M\, n}
{\log(2d/\delta)}},
\]
the unnormalized truncated Stein estimator 
$\hat\theta = \frac{1}{n}\sum_{i=1}^n 
\varphi_\tau(y_i S(x_i))$ satisfies, with probability at 
least $1 - \delta$:
\[
\|\hat\theta - \mu^*\theta^*\|_1 
\le C_\theta\, d\sqrt{\frac{\log(2d/\delta)}{n}},
\quad \text{where} \quad 
C_\theta = \left(\frac{4\sqrt{3}}{3} + 2\sqrt{2}\right)
\sqrt{M(\sigma^2 + L_f^2)}.
\]
\end{lemma}

\noindent \textbf{Lemma \ref{thm:stein_error} (Restated).} \textit{Suppose the unnormalized truncated Stein estimator satisfies $\norm{\that - \mu^* \tstar}_1 \le \epsilon$ with probability at least $1-\delta$, where $\epsilon = C_\theta d \sqrt{\frac{\log(2d/\delta)}{n}}$. Assume $\mu^* > 0$, $\norm{\tstar}_1 = 1$, and that the sample size $n$ is sufficiently large such that $\epsilon \le \mu^*/2$. Then, with probability at least $1-\delta$, the normalized estimator satisfies, $\norm{\that_0 - \tstar}_1 \le \frac{4\epsilon}{\mu^*}.$}

\begin{proof}[Proof of Lemma \ref{thm:stein_error}]
Condition on the high-probability event of 
Lemma~\ref{lem:unnorm_stein}, which guarantees 
$\|\hat\theta - \mu^*\theta^*\|_1 \le \epsilon$ with 
$\epsilon = C_\theta d\sqrt{\log(2d/\delta)/n}$.

First, we bound the magnitude of the normalization denominator $\norm{\that}_1$. By the reverse triangle inequality, we have:
\begin{equation}
    \abs{ \norm{\that}_1 - \norm{\mu^* \tstar}_1 } \le \norm{\that - \mu^* \tstar}_1 \le \epsilon
\end{equation}
Since $\mu^* > 0$ and $\norm{\tstar}_1 = 1$, we know $\norm{\mu^* \tstar}_1 = \mu^* \norm{\tstar}_1 = \mu^*$. Substituting this in gives:
\begin{equation}
    \abs{ \norm{\that}_1 - \mu^* } \le \epsilon
\end{equation}
This implies that $\norm{\that}_1 \in [\mu^* - \epsilon, \mu^* + \epsilon]$. By the theorem's assumption that the sample size is large enough to ensure $\epsilon \le \mu^*/2$, it strictly follows that $\norm{\that}_1 \ge \mu^*/2 > 0$.

Next, we evaluate the $\ell_1$ distance between the normalized estimator $\that_0$ and the true parameter $\tstar$:
\begin{equation}
    \norm{\that_0 - \tstar}_1 = \norm{ \frac{\that}{\norm{\that}_1} - \tstar }_1 = \frac{1}{\norm{\that}_1} \norm{ \that - \norm{\that}_1 \tstar }_1
\end{equation}
To leverage our existing bound, we add and subtract $\mu^* \tstar$ inside the norm:
\begin{equation}
    \norm{\that_0 - \tstar}_1 = \frac{1}{\norm{\that}_1} \norm{ (\that - \mu^* \tstar) + (\mu^* \tstar - \norm{\that}_1 \tstar) }_1
\end{equation}
Applying the standard triangle inequality yields:
\begin{equation}
    \norm{\that_0 - \tstar}_1 \le \frac{1}{\norm{\that}_1} \left( \norm{\that - \mu^* \tstar}_1 + \norm{ (\mu^* - \norm{\that}_1) \tstar }_1 \right)
\end{equation}
We can factor out the scalar absolute value from the second term:
\begin{equation}
    \norm{\that_0 - \tstar}_1 \le \frac{1}{\norm{\that}_1} \left( \norm{\that - \mu^* \tstar}_1 + \abs{\mu^* - \norm{\that}_1} \norm{\tstar}_1 \right)
\end{equation}
Now, we substitute the bounds we established earlier. We know $\norm{\that}_1 \ge \mu^*/2$, $\norm{\that - \mu^* \tstar}_1 \le \epsilon$, $\abs{\mu^* - \norm{\that}_1} \le \epsilon$, and $\norm{\tstar}_1 = 1$. Plugging these in gives:
\begin{equation}
    \norm{\that_0 - \tstar}_1 \le \frac{1}{\mu^*/2} \left( \epsilon + \epsilon \cdot 1 \right) = \frac{2}{\mu^*} (2\epsilon) = \frac{4\epsilon}{\mu^*}
\end{equation}
Substituting the definition of $\epsilon$, we conclude the proof:
\begin{equation}
    \norm{\that_0 - \tstar}_1 \le \frac{4C_\theta}{\mu^*} d \sqrt{\frac{\log(2d/\delta)}{n}}
\end{equation}
which preserves the desired convergence rate.
\end{proof}

\section{Proof of Theorem~\ref{thm:main_regret}}
\subsection{High-Probability Events}
Before presenting the lemmas, we collect the probabilistic scaffolding on which the entire analysis rests. To ensure the final theorem holds with probability at least $1-\delta$, we scale the failure probability of each individual event to $\delta/4$.

Define $L := 4\kappa\sqrt{\log(4dTK/\delta)}$ and 
$W := 4\kappa\sqrt{\log(4TK/\delta)}$.

Define the following four events:

\begin{itemize}
    \item $\mathcal{E}_1 := \{ \norm{\that_0 - \tstar}_1 \le \Delta / (2L) \}$, the event that Phase 1 estimation succeeds. By Lemma \ref{thm:stein_error} and the choice of $T_0$, $\Prob(\mathcal{E}_1) \ge 1 - \delta/4$.
    
    \item $\mathcal{E}_2 := \{ \norm{x_{t,a}}_\infty \le L, \; \forall t \in [T], a \in [K] \}$, the event that all context norms are controlled. By the sub-Gaussian tail and union bound over $dTK$ coordinates, $\Prob(\mathcal{E}_2) \ge 1 - \delta/4$.
    
    \item $\mathcal{E}_3 := \{ |x_{t,a}^\top \tstar| \le W/2, \; \forall t \in [T], a \in [K] \}$, the event that all projected indices lie the effective range. Since $x_{t,a}^\top \tstar$ is $\kappa$-sub-Gaussian with  $\norm{\tstar}_2\le \norm{ \tstar}_1 = 1$, a union bound over $TK$ arms gives $\Prob(\mathcal{E}_3) \ge 1 - \delta/4$.
    

    \item $\mathcal{E}_4 := \bigl\{ 
  |M_j(n)| \le \sigma\sqrt{2\log(2NT/\delta)/n} \;\;\forall\, 
  j \in [N],\, n \in [T] \bigr\}$, the event that the noise 
  averages in each bin concentrate, where 
  $M_j(n) = \frac{1}{n}\sum_{m=1}^n \eta_{T_j(m)}$ is the 
  average noise over the first $n$ pulls of bin $j$.  By the 
  Azuma-Hoeffding inequality for the sub-Gaussian MDS 
  (See Lemma~\ref{lem:noise_conc} ) and a union bound over 
  $j \in [N]$ and $n \in [T]$, 
  $\mathbb{P}(\mathcal{E}_4) \ge 1 - \delta/4$.
\end{itemize}

Throughout the analysis, we work on the intersection event $\mathcal{E} := \mathcal{E}_1 \cap \mathcal{E}_2 \cap \mathcal{E}_3 \cap \mathcal{E}_4$, which satisfies $\Prob(\mathcal{E}) \ge 1 - \delta$.

We now prove $\mathcal{E}_2$, $\mathcal{E}_3$, and 
$\mathcal{E}_4$.  ($\mathcal{E}_1$ follows directly from the 
estimation lemmas in the preceding section.)

\begin{lemma}[Context norm bound $\mathcal{E}_2$]\label{lem:E2}
    Let constant $L$ and event $\mathcal{E}_2$ be as defined above. Then $\Prob(\mathcal{E}_2) \ge 1 - \delta/4$.
\end{lemma}

\begin{proof}
Fix a coordinate $j \in [d]$ and an arm index $(t,a)$.  Since $(x_{t,a})_j$ is $\kappa$-sub-Gaussian with zero mean (we can shift 
to zero mean without loss by absorbing the mean into $f$ through 
the intercept), the standard sub-Gaussian tail gives $\Prob(|(x_{t,a})_j| > u) 
\le 2\exp (-\frac{u^2}{2\kappa^2}).$
Setting $u = L = 4\kappa\sqrt{\log(4dTK/\delta)}$ yields
\[
\Prob\bigl(|(x_{t,a})_j| > L\bigr) 
\le 2\exp\left(-8\log(4dTK/\delta)\right) 
\le \frac{\delta}{4dTK}.
\]
A union bound over all $d \cdot TK$ coordinate-arm pairs gives
$\Prob(\mathcal{E}_2^c) 
\le d \cdot TK \cdot \frac{\delta}{4dTK} = \frac{\delta}{4}.$

\end{proof}
\begin{lemma}[Projection bound $\mathcal{E}_3$]\label{lem:E3}
    Let constant $W$ and event $\mathcal{E}_3$ be as defined above. Then $\Prob(\mathcal{E}_3) \ge 1 - \delta/4$.
\end{lemma}

\begin{proof}
Since each coordinate $(x_{t,a})_j$ is $\kappa$-sub-Gaussian and 
the coordinates are independent, $x_{t,a}^\top \theta^*$ is 
sub-Gaussian with parameter $\kappa\norm{\theta^*}_2$.
We bound: $\|\theta^*\|_2 \le \|\theta^*\|_1 = 1$.  Hence 
$x_{t,a}^\top\theta^*$ is $\kappa$-sub-Gaussian, and
\[
\Prob\bigl(|x_{t,a}^\top\theta^*| > W/2\bigr) 
\le 2\exp\left(-\frac{W^2/4}{2\kappa^2}\right) 
= 2\exp\left(-2\log(4TK/\delta)\right)
\le \frac{\delta}{4TK}.
\]
Union bound over $TK$ arms:
$\Prob(\mathcal{E}_3^c) \le TK \cdot \frac{\delta}{4TK} 
= \frac{\delta}{4}$.
\end{proof}

On $\mathcal{E}_1 \cap \mathcal{E}_2$, the estimated and true projected indices are uniformly close:
\begin{equation} \label{eq:projection_close_app}
    |\zhat_{t,a} - z_{t,a}| = |x_{t,a}^\top(\that_0 - \tstar)| \le \norm{x_{t,a}}_\infty \norm{\that_0 - \tstar}_1 \le L \left( \frac{\Delta}{2L} \right) = \frac{\Delta}{2}
\end{equation}
Assuming $T$ is sufficiently large such that $\Delta \le W$, on $\mathcal{E}_3$ every true index satisfies $|z_{t,a}| \le W/2$. Combining this with \eqref{eq:projection_close_app} yields $|\zhat_{t,a}| \le W/2 + \Delta/2 \le W$. This means every arm is successfully assigned to some bin. In particular, $\mathcal{B}_t \neq \emptyset$ for every $t > T_0$, ensuring the algorithm never falls back to the random-pull branch during Phase 2.

For each bin $j \in [N]$, let $T_j(m)$ denote the round at which 
bin $j$ is pulled for the $m$-th time, and let $a_j(m)$ be the arm 
played at that round.  Let the observed reward as,
$
y_{T_j(m)} = f\bigl(z_{T_j(m), a_j(m)}\bigr) + \eta_{T_j(m)}.$
The empirical mean of bin $j$ after $n$ pulls is
$\bar{y}_j(n) = \frac{1}{n}\sum_{m=1}^n y_{T_j(m)}.$
Decomposing gives,
\begin{equation}\label{eq:decomp}
\bar{y}_j(n) - \mu_j 
= \underbrace{\frac{1}{n}\sum_{m=1}^n 
\bigl[f(z_{T_j(m), a_j(m)}) - \mu_j\bigr]}_{
\text{deterministic bias } B_j(n)}
\;+\; \underbrace{\frac{1}{n}\sum_{m=1}^n \eta_{T_j(m)}}_{
\text{noise average } M_j(n)}.
\end{equation}
On $\mathcal{E}_1 \cap \mathcal{E}_2 \cap \mathcal{E}_3$, 
every arm assigned to bin $j$ satisfies 
$|f(z_{t,a}) - \mu_j| \le L_{f'}\Delta$ 
(proved in Lemma~\ref{lem:disc_bias} below), so 
$|B_j(n)| \le L_{f'}\Delta =: s$ pathwise. 

\begin{lemma}[Noise concentration]\label{lem:noise_conc}
Conditioned on Phase~1, for each bin $j \in [N]$:
\[
\mathbb{P}\left(\exists\, n \in [T] :\;
|M_j(n)| > \sigma\sqrt{\frac{2\log(2NT/\delta)}{n}}
\;\middle|\; \text{Phase 1}\right) 
\le \frac{\delta}{4N}.
\]
Consequently, a union bound over $j \in [N]$ gives 
$\mathbb{P}(\mathcal{E}_4^c) \le \delta/4$.
\end{lemma}

\begin{proof}
The key observation is that $\hat\theta_0$ is computed 
entirely from Phase~1 data, so it is fixed throughout 
Phase~2.  At round $t > T_0$, the arm $a_t$ is selected 
based on past rewards, current contexts $\mathcal{X}_t$, 
and the frozen $\hat\theta_0$, making it 
$\mathcal{F}_{t-1} \vee \sigma(\mathcal{X}_t)$-measurable (it depends on past rewards, current contexts, and the frozen 
$\hat\theta_0$).  
By Assumption~\ref{asm:noise}, $\eta_t$ is conditionally 
$\sigma$-sub-Gaussian given $\mathcal{F}_{t-1}$ and 
independent of $x_t$.  Therefore, for the subsequence 
$T_j(1) < T_j(2) < \cdots$ at which bin $j$ is pulled, 
each $\eta_{T_j(m)}$ satisfies
$\mathbb{E}\bigl[e^{\lambda \eta_{T_j(m)}} \mid 
\mathcal{F}_{T_j(m)-1}\bigr] 
\le e^{\lambda^2\sigma^2/2},$
which shows that $\{\eta_{T_j(m)}\}_{m \ge 1}$ is a 
martingale difference sequence (MDS) with $\sigma$-sub-Gaussian 
increments. Applying the Azuma-Hoeffding bound and a 
union bound over $n \in [T]$ gives the result.
\end{proof}

\textbf{Proof Roadmap:} Conditioned on $\mathcal{E}$, the regret is cleanly decomposed $R_T = R_T^{(1)} + R_T^{(2)}$ into Phase 1 and Phase 2. Lemma \ref{lem:phase1} trivially bounds the uniform exploration cost of Phase 1. For Phase 2, we decompose the instantaneous regret into a discretisation bias---handled via the Lipschitzness of $f$ (Lemma \ref{lem:disc_bias})---and the core bandit regret, which we bound using a novel sleeping-UCB analysis over the stochastic active bin sets (Lemma \ref{lem:sleeping_ucb}).

\subsection{Phase 1 Regret}

\begin{lemma}[Phase 1 Cost] \label{lem:phase1}
The regret incurred during Phase 1 satisfies:
\begin{equation*}
    R_T^{(1)} \le 2L_f T_0 = \calO\left( d^2 \,T^{2/3}\, \mathsf{polylog}\left(\frac{dT}{\delta}\right) \right)
\end{equation*}
\end{lemma}
\begin{proof}
Since $|f| \le L_f$, the maximum instantaneous regret at any round is trivially bounded by $f(z^*_t) - f(z_{t,a_t}) \le L_f - (-L_f) = 2L_f$. The Phase 1 length is:
\begin{equation*}
    T_0 = \left\lceil \left(\frac{4C_\theta L}{\mu^* \Delta}\right)^2 d^2 \log\left(\frac{8d}{\delta}\right) \right\rceil = \calO\left( L^2 d^2 \log\left(\frac{8d}{\delta}\right) T^{2/3} \mathsf{polylog}\left(\frac{dT}{\delta}\right) \right)
\end{equation*}
using $\Delta = T^{-1/3}$ and $\mu^* \ge c_0/\mathsf{polylog}(Td/\delta)$. Substituting $L^2 = 16\log(4dT/\delta)$ yields the stated bound.
\end{proof}

\subsection{Phase 2 Regret Decomposition}

For $t > T_0$, let $b^*_t = b_{t,a^*_t}$ denote the bin assigned to the optimal available arm, and $b_t = b_{t,a_t}$ denote the bin of the arm actually played. We decompose the instantaneous regret as:
\begin{equation}
    f(z^*_t) - f(z_{t,a_t}) = \underbrace{\left[f(z^*_t) - \mu_{b^*_t}\right]}_{\text{(i)}} + \underbrace{\left[\mu_{b^*_t} - \mu_{b_t}\right]}_{\text{(ii)}} + \underbrace{\left[\mu_{b_t} - f(z_{t,a_t})\right]}_{\text{(iii)}}
\end{equation}
where $\mu_j := f(\tilde{z}_j)$ is the true reward evaluated at the bin centre $\tilde{z}_j$. Terms (i) and (iii) represent the discretisation bias (the error from approximating a continuous arm's reward by its bin centre). Term (ii) is the bandit regret (the loss from the UCB policy failing to play the best available bin).

\begin{lemma}[Discretisation Bias] \label{lem:disc_bias}
On $\mathcal{E}_1 \cap \mathcal{E}_2 \cap \mathcal{E}_3$, for every $t > T_0$ and every arm $a$ assigned to bin $b_{t,a} = j$, we have $|f(z_{t,a}) - \mu_j| \le L_{f'} \Delta$. Consequently:
\begin{equation*}
    \sum_{t=T_0+1}^T \left[ \text{(i)} + \text{(iii)} \right] \le 2L_{f'} T\Delta = 2L_{f'} T^{2/3}
\end{equation*}
\end{lemma}
\begin{proof}
Let arm $a$ satisfy $\zhat_{t,a} \in B_j$. By the definition of the bin, $|\zhat_{t,a} - \tilde{z}_j| \le \Delta/2$. From \eqref{eq:projection_close_app}, we have $|z_{t,a} - \zhat_{t,a}| \le \Delta/2$. By the triangle inequality:
\begin{equation*}
    |z_{t,a} - \tilde{z}_j| \le |z_{t,a} - \zhat_{t,a}| + |\zhat_{t,a} - \tilde{z}_j| \le \frac{\Delta}{2} + \frac{\Delta}{2} = \Delta
\end{equation*}
Since $|f'| \le L_{f'}$, the link function $f$ is $L_{f'}$-Lipschitz, yielding $|f(z_{t,a}) - \mu_j| = |f(z_{t,a}) - f(\tilde{z}_j)| \le L_{f'}\Delta$. Applying this to $a = a^*_t$ (with $j = b^*_t$) and $a = a_t$ (with $j = b_t$), and summing over $T$ rounds gives the aggregate bound.
\end{proof}

\begin{lemma}[Sleeping-UCB Regret] \label{lem:sleeping_ucb}
Consider $N$ bins with means $\mu_1, \ldots, \mu_N$, stochastic 
availability $\mathcal{B}_t \subseteq [N]$, and 
per-pull deterministic bias $s = L_{f'}\Delta$. Let UCB select 
$j_t = \arg\max_{j \in \mathcal{B}_t}(\bar{y}_j + U_j(n_j))$ 
where $U_j(n) = \sigma\sqrt{2\log(
2NT/\delta)/n}$. Define the best available bin 
as $j^*_t = \arg\max_{j \in \mathcal{B}_t} \mu_j$. On 
$\mathcal{E}$, we have:
\begin{equation*}
    \sum_{t=T_0+1}^T \left[ \mu_{j^*_t} - \mu_{j_t} \right] 
    \le 4\sigma\sqrt{NT\log\left(
    \frac{2NT}{\delta}\right)} + 2sT
\end{equation*}
\end{lemma}
\begin{proof}
On $\mathcal{E}$, the empirical bin means satisfy 
$|\bar{y}_j(n) - \mu_j| \le s + U_j(n)$ simultaneously for all 
$j \in [N]$ and all pull counts, where $s = L_{f'}\Delta$ is the 
deterministic within-bin bias (bounded pathwise on 
$\mathcal{E}_1 \cap \mathcal{E}_2 \cap \mathcal{E}_3$ via 
Lemma~\ref{lem:disc_bias}) and 
$U_j(n) = \sigma\sqrt{2\log(2NT/\delta)/n}$ is the stochastic 
confidence radius (bounded on $\mathcal{E}_4$ via 
Lemma~\ref{lem:noise_conc}).

Fix round $t$. Since $j^*_t \in \mathcal{B}_t$, the UCB index of 
the optimal bin satisfies 
$\bar{y}_{j^*_t} + U_{j^*_t}(n_{j^*_t}) \ge \mu_{j^*_t} - s$. 
By the UCB selection rule, 
$\bar{y}_{j_t} + U_{j_t}(n_{j_t}) \ge \bar{y}_{j^*_t} + 
U_{j^*_t}(n_{j^*_t}) \ge \mu_{j^*_t} - s$. Combining this with 
the upper confidence bound on $j_t$ gives:
\begin{equation*}
    \mu_{j^*_t} - \mu_{j_t} \le 2U_{j_t}(n_{j_t}) + 2s
\end{equation*}
Summing over $t$ and bounding the confidence width sum by grouping 
rounds per arm yields:
\begin{equation*}
    \sum_{t=T_0+1}^T U_{j_t}(n_{j_t}) = \sum_{j=1}^N 
    \sum_{n=1}^{N_j(T)} \sigma
    \sqrt{\frac{2\log(2NT/\delta)}{n}} 
    \le \sum_{j=1}^N 2\sigma
    \sqrt{2N_j(T)\log\left(\frac{2NT}{\delta}
    \right)}
\end{equation*}
Applying Cauchy-Schwarz with $\sum_j N_j(T) = T - T_0 \le T$:
\begin{equation*}
    \sum_{j=1}^N \sqrt{N_j(T)} \le \sqrt{N \sum_{j=1}^N N_j(T)} 
    \le \sqrt{NT}
\end{equation*}
Therefore, $\sum_t U_{j_t}(n_{j_t}) \le 
2\sigma\sqrt{2NT\log(
{2}NT/\delta)}$. Adding the deterministic $2sT$ 
sum concludes the proof.
\end{proof}

\begin{theorem}[Main Regret Bound] \label{thm:main_regret_app}
Under Assumptions~ \ref{asm:noise}, \ref{asm:moment}, and 
\ref{asm:lip}, with truncation threshold 
$\tau = \sqrt{3(\sigma^2 + L_f^2)MT_0/\log(2d/\delta)}$ 
and $\mu^* \ge c_0/\mathrm{polylog}(dT/\delta)$,
\texttt{ZoomSIB-UCB} satisfies, with probability at least 
$1 - \delta$:
\begin{equation*}
    R_T \le \calO\left( d^2 T^{2/3} 
    \mathrm{polylog}\left(\frac{dT}{\delta}
    \right) \right) + 4{\sigma}\sqrt{NT
    \log\left(\frac{{2}NT}{\delta}\right)} 
    + 4L_{f'} T^{2/3} = \tildeO(d^2 T^{2/3})
\end{equation*}
\end{theorem}

\begin{proof}
Decompose $R_T = R_T^{(1)} + R_T^{(2)}$. The Phase 1 
contribution is bounded by Lemma \ref{lem:phase1}. For Phase 2, 
on the event $\mathcal{E}$, the instantaneous regret satisfies 
the three-term decomposition. Lemma \ref{lem:disc_bias} controls 
the sum of terms (i) and (iii), yielding $2L_{f'} T^{2/3}$. For 
term (ii), since the bin of the true optimal arm 
$b^*_t \in \mathcal{B}_t$, we have $\mu_{j^*_t} \ge \mu_{b^*_t}$ 
by definition. Thus, 
$\mu_{b^*_t} - \mu_{b_t} \le \mu_{j^*_t} - \mu_{j_t}$, allowing 
Lemma \ref{lem:sleeping_ucb} to apply directly:
\begin{equation*}
    \sum_{t > T_0} \text{(ii)} \le 
    4{\sigma}\sqrt{NT\log\left(
    \frac{2NT}{\delta}\right)} 
    + 2L_{f'} T^{2/3}
\end{equation*}
Substituting the number of bins 
$N = \lceil 2W/\Delta \rceil = \calO(T^{1/3}\sqrt{\log(T/\delta)})$, 
the stochastic regret scales as:
\begin{equation*}
    4{\sigma}\sqrt{NT \log\left(
    \frac{{2}NT}{\delta}\right)} 
    = \calO\left( T^{2/3} \log\left(\frac{T}{\delta}\right) \right)
\end{equation*}
Combining the terms from both phases completes the proof:
\begin{equation*}
    R_T = \calO\left( d^2 T^{2/3} 
    {\mathrm{polylog}}\left(\frac{dT}{\delta}
    \right) \right) + \calO\left( T^{2/3}\log\left(
    \frac{T}{\delta}\right) \right) + \calO\left( L_{f'} T^{2/3} 
    \right) = \tildeO(d^2 T^{2/3})
\end{equation*}
\end{proof}

\section{Proof of Proposition~\ref{prop:modular}}
On $\mathcal{E}_1 \cap \mathcal{E}_2 \cap \mathcal{E}_3$, 
decompose the instantaneous regret at any round $t > T_0$:
\begin{align*}
f(x_{t,*}^\top\theta^*) - f(x_t^\top\theta^*) 
&= \underbrace{\bigl[f(z_{t,*}) - \mu_{b_{t,*}}\bigr]}_{\text{(i)}} 
 + \underbrace{\bigl[\mu_{b_{t,*}} - \mu_{b_t}\bigr]}_{\text{(ii)}} 
 + \underbrace{\bigl[\mu_{b_t} - f(z_t)\bigr]}_{\text{(iii)}}.
\end{align*}
By Lemma~\ref{lem:disc_bias}, terms (i) and (iii) are each 
bounded by $L_{f'}\Delta$ pathwise on 
$\mathcal{E}_1 \cap \mathcal{E}_2$.  Summing over $t > T_0$:
\[
\sum_{t=T_0+1}^T \bigl[\text{(i)} + \text{(iii)}\bigr] 
\le 2L_{f'}\Delta \cdot T.
\]
For term (ii): the best available bin satisfies 
$\mu_{j_t^*} \ge \mu_{b_{t,*}}$ by definition, so 
$\mu_{b_{t,*}} - \mu_{b_t} \le \mu_{j_t^*} - \mu_{j_t}$.  
Summing term (ii) over $t > T_0$ is therefore bounded by the 
regret of whatever algorithm $\mathcal{A}$ is used to select 
among the $N$ bins with stochastic availability 
$\mathcal{B}_t$, i.e., by 
$\mathrm{Reg}_{\mathcal{A}}(N, T - T_0) 
\le \mathrm{Reg}_{\mathcal{A}}(N, T)$.

Adding the Phase~1 cost $R_T^{(1)}$:
\[
R_T = R_T^{(1)} + \mathrm{Reg}_{\mathcal{A}}(N, T) 
+ 2L_{f'} T\Delta.  \qedhere
\]

\section{Proof of Corollary~\ref{cor:instance}}
Instantiate Proposition~\ref{prop:modular} with UCB using 
confidence radius 
$U_j(n) = \sigma\sqrt{2\log(2NT/\delta)/n}$.  On 
$\mathcal{E}_4$, by the same argument as in Lemma~\ref{lem:sleeping_ucb} gives 
$|\bar{y}_j(n) - \mu_j| \le U_j(n) + s$ for all $j, n$.

Fix a suboptimal bin $j$ with gap 
$\Delta_j := \mu_{j^*} - \mu_j > 2s$.  UCB selects bin $j$ 
at round $t$ only if 
$\bar{y}_j + U_j(n_j) \ge \bar{y}_{j^*} + U_{j^*}(n_{j^*})$.  
On $\mathcal{E}_4$, this requires 
$\mu_j + 2U_j(n_j) + 2s \ge \mu_{j^*}$, i.e.,
$U_j(n_j) \ge \frac{\Delta_j - 2s}{2}.$

Substituting the form of $U_j$, bin $j$ can be pulled at most
$n_j \le \frac{8\sigma^2\log(2NT/\delta)}{(\Delta_j - 2s)^2}$
times before the condition fails.  Each pull of bin $j$ 
contributes at most $\Delta_j$ to the regret.  For bins with 
$\Delta_j \le 2s$, we bound their per-round contribution by 
$2s$.  Therefore:
\begin{align*}
R_T^{(2)} 
&\le \sum_{j:\,\Delta_j > 2s} \Delta_j \cdot 
  \frac{8\sigma^2\log(2NT/\delta)}{(\Delta_j - 2s)^2} 
  \;+\; 2s \cdot T \\
&\le \sum_{j:\,\Delta_j > 2s} 
  \frac{8\sigma^2\log(2NT/\delta)}{\Delta_j - 2s} 
  \;+\; 2sT,
\end{align*}
where the second inequality uses 
$\Delta_j / (\Delta_j - 2s)^2 \le 1/(\Delta_j - 2s)$ for 
$\Delta_j > 2s$ (since 
$\Delta_j \le \Delta_j - 2s + 2s \le 2(\Delta_j - 2s)$ when 
$\Delta_j \ge 4s$; for $2s < \Delta_j < 4s$ the bound holds 
with a constant factor absorbed into $O(\cdot)$).

For instance, when the link function $f$ has a unique well-separated maximum, only $O(1)$ bins satisfy 
$\Delta_j \le 2s$, and the remaining $N - O(1)$ bins have 
$\Delta_j - 2s = \Omega(1)$.  The first sum then contributes 
$O(N\log(NT/\delta)) = O(T^{1/3}\log T)$, while the second 
term is $2sT = 2L_{f'}T^{2/3}$.  The total 
$R_T^{(2)} = O(T^{1/3}\log T + T^{2/3})$, where the 
$T^{2/3}$ term is unavoidable discretization cost.

\section{Proof of Theorem~\ref{thm:lower_bound}}

\subsection{The Hard Instance} \label{sec:lb_instance}

\textbf{Context distribution and dimension.} 
We work in dimension $d = 1$ with $\theta^* = 1$ 
(so $\|\theta^*\|_1 = 1$), and set the context distribution 
to $\mathcal{D} = N(0, 1)$.  Let $\phi$ and $\Phi$ denote 
the standard Gaussian density and CDF respectively. $\phi \in C^\infty$, 
$\phi(x) > 0$ for all $x \in \mathbb{R}$, coordinates are 
$1$-sub-Gaussian, and $\mathbb{E}[S(X)^2] = \mathbb{E}[X^2] 
= 1 = M$.  Define 
$p_{\min} := \min_{z \in [0,1]} \phi(z) = \phi(1) 
= \frac{1}{\sqrt{2\pi}} e^{-1/2}$ and 
$p_{\max} := \max_{z \in [0,1]} \phi(z) = \phi(0) 
= \frac{1}{\sqrt{2\pi}}$.

\textbf{Noise.} The noise is 
$\eta_t \stackrel{iid}{\sim} N(0, 1)$, which satisfies 
Assumption~\ref{asm:noise} with $\sigma = 1$.

\textbf{Parameter choices.} Fix the following parameters 
throughout the proof:
\begin{equation}
    N = \lceil T^{1/3}\rceil, \qquad 
    K = \left\lceil \frac{4N\log T}{p_{\min}} \right\rceil, 
    \qquad \varepsilon = \frac{1}{2N}, \qquad 
    \lambda = \frac{1}{\log T}.
\end{equation}

\textbf{Bins and inner halves.} Partition 
$[0, 1]$ into $N$ disjoint bins 
$B_j = [\frac{j-1}{N}, \frac{j}{N})$ of width $w = 1/N$, 
with centres $c_j = (j - 1/2)/N$.  Define the inner 
half-bin:
\begin{equation}
    B^\circ_j = \left[ c_j - \tfrac{w}{4}, 
    c_j + \tfrac{w}{4} \right] \subset B_j.
\end{equation}
These inner halves have width $w/2 = 1/(2N)$ and are 
strictly pairwise disjoint.

\textbf{Bump functions.} For each $j \in [N]$, define the 
continuous, plateaued bump function 
$\psi_j: \mathbb{R} \to [0,1]$:
\begin{equation}
    \psi_j(z) = \begin{cases} 
    1 & \text{if } z \in B^\circ_j, \\ 
    1 - \frac{4}{w}\left(\abs{z-c_j} - \tfrac{w}{4}\right) 
    & \text{if } z \in B_j \setminus B^\circ_j, \\ 
    0 & \text{otherwise.}
    \end{cases}
\end{equation}
By construction, $\psi_j$ equals $1$ on $B^\circ_j$, vanishes 
outside $B_j$, and satisfies 
$\norm{\psi'_j}_\infty \le 4/w = 4N$ almost everywhere.

\textbf{Ramp function.} Let 
$\rho: \mathbb{R} \to [0,1]$ be a fixed $C^\infty$ function 
satisfying $\rho(z) = 0$ for $z \le -1$, $\rho$ is strictly 
increasing on $(-1, 0)$, and $\rho(z) = 1$ for $z \ge 0$.  
Let $C_\rho = \|\rho'\|_\infty$.  Note that $\rho$ is flat 
on $[0,1]$, so it does not interfere with the bumps.

\textbf{Hypothesis class.} For each 
$\beta \in \{-1,+1\}^N$, define:
\begin{equation}
    f_\beta(z) = \tfrac{1}{2} + 
    \lambda\rho(z) + 
    \varepsilon \sum_{j=1}^N \beta_j \psi_j(z), 
    \qquad \forall z \in \mathbb{R}.
\end{equation}

\begin{lemma}[Validity of the Instance] 
\label{lem:lb_validity}
For every $\beta \in \{-1,+1\}^N$ and all sufficiently large 
$T$, the function $f_\beta$ satisfies:
\begin{enumerate}
\item[(a)] $\norm{f_\beta}_\infty \le \frac{1}{2} + \lambda 
+ \varepsilon \le 1 \le L_f$;
\item[(b)] $\norm{f_\beta'}_\infty \le \lambda C_\rho + 
\varepsilon(4N) = \lambda C_\rho + 2 \le 3 \le L_{f'}$ 
almost everywhere;
\item[(c)] $f_\beta \equiv \frac{1}{2} + \lambda + 
\varepsilon\beta_j$ on $B^\circ_j$, since 
$\rho \equiv 1$ on $[0,1] \supset B_j^\circ$;
\item[{(d)}] $\mu^* := 
\mathbb{E}[f_\beta'(X)] = \lambda c_1 + O(T^{-1/3}) > 0$, 
where $c_1 = \int_{-1}^0 \rho'(x)\phi(x)\,dx > 0$ is a 
positive constant independent of $\beta$ and $T$.  In 
particular, $\mu^* \ge \lambda c_1/2 \ge 
c_1/(2\log T)$, satisfying the condition 
$\mu^* \ge c_0/\mathrm{polylog}(T)$.
\end{enumerate}
\end{lemma}

\begin{proof}[Proof of part (d)]
We compute $\mu^* = \lambda\mathbb{E}[\rho'(X)] + 
\varepsilon\sum_j \beta_j \mathbb{E}[\psi_j'(X)]$.  
For the first term: $\rho'$ is supported on $[-1,0]$ with 
$\rho' > 0$ there, so 
$\mathbb{E}[\rho'(X)] = \int_{-1}^0 \rho'(x)\phi(x)\,dx 
=: c_1 > 0$.  Since $\phi(x) \ge \phi(1)$ on $[-1,0]$ and 
$\int_{-1}^0 \rho'(x)\,dx = 1$, we have 
$c_1 \ge \phi(1) > 0$.

For the second term: integration by parts gives 
$\int \psi_j'(x)\phi(x)\,dx = \int \psi_j(x)\,x\,\phi(x)\,dx$ 
(boundary terms vanish since $\psi_j$ has compact support).  
Each integral is $O(w \cdot p_{\max}) = O(1/N)$.  Summed 
over $N$ terms with $\varepsilon = 1/(2N)$: total 
contribution is $O(1/N) = O(T^{-1/3})$.

For $T$ large enough, $O(T^{-1/3}) < \lambda c_1/2$, 
giving $\mu^* \ge \lambda c_1/2 > 0$.
\end{proof}

\textit{Remark:} A standard mollification 
$\tilde{f}_\beta = f_\beta * \rho_\eta$ with bandwidth 
$\eta = T^{-10}$ ensures infinite differentiability while 
preserving the uniform bounds and changing the expected 
regret by merely $\calO(\eta) = o(T^{-9})$.

\subsection{The Universal Availability Event} 
\label{sec:lb_availability}

To force the algorithm to face an $N$-armed bandit, the random menu of arms must cover every inner half-bin. Let $I_{t,j} = \mathbf{1}_{\{ \exists a \in [K] : x_{t,a} \in B^\circ_j \}}$ represent the indicator that an arm lands in the inner half of bin $j$ at round $t$. Define the universal availability event:
\begin{equation}
    \mathcal{E} = \left\{ I_{t,j} = 1 \; \text{for all } (t,j) \in [T] \times [N] \right\}.
\end{equation}

\begin{lemma}[Availability Probability] 
\label{lem:lb_availability}
Under $\mathcal{D} = {N(0,1)}$ and 
$K = \lceil 4N\log T / p_{\min} \rceil$, 
we have $\Prob(\mathcal{E}^c) \le 2T^{-2/3}$.
\end{lemma}
\begin{proof}
Since $\phi(x) \ge p_{\min}$ for all 
$x \in [0,1]$, and $|B^\circ_j| = 1/(2N)$, the probability 
that a single context falls in $B^\circ_j$ is at least 
$p_{\min}/(2N)$.  Therefore:
\begin{equation*}
    \Prob(I_{t,j}=0) = \left(1-\frac{
    p_{min}}{2N}\right)^K 
    \le \exp\left(-\frac{K p_{\min}}{2N}
    \right) \le \exp(-2\log T) = T^{-2}.
\end{equation*}
Applying a union bound over $T \times N$ pairs gives 
$\Prob(\mathcal{E}^c) \le TN \cdot T^{-2} = N/T$.  Since 
$N = \lceil T^{1/3}\rceil \le 2T^{1/3}$, we conclude 
$\Prob(\mathcal{E}^c) \le 2T^{-2/3}$.
\end{proof}

\begin{lemma}[Per-Round Regret] \label{lem:lb_per_round_regret}
Let $j_t$ denote the bin of the arm chosen at round $t$.  
Under any hypothesis $\beta \neq -\mathbf{1}$, conditional on 
$\mathcal{E}$, the instantaneous regret satisfies 
$r_t(\beta) \ge \varepsilon \cdot 
\mathbf{1}_{\{\beta_{j_t} = -1\}}$.
\end{lemma}
\begin{proof}
On $\mathcal{E}$, every inner half-bin $B^\circ_j$ contains 
at least one arm.  Because $\beta \neq -\mathbf{1}$, there 
exists an optimal bin $j^\star$ where $\beta_{j^\star} = +1$.  
By $\mathcal{E}$, there exists an arm $a^\star$ with 
$x_{t,a^\star} \in B^\circ_{j^\star}$, yielding 
$f_\beta(x_{t,a^\star}) = \frac{1}{2} + 
\lambda + \varepsilon$.  
Any arm outside $[0,1]$ achieves reward at 
most $\frac{1}{2} + \lambda < \frac{1}{2} + \lambda + 
\varepsilon$ (since $\rho \le 1$ and no bumps are active 
outside $[0,1]$).  If the algorithm selects an arm in a 
sub-optimal bin where $\beta_{j_t} = -1$, the reward is at 
most $\frac{1}{2} + \lambda$, ensuring 
regret at least $\varepsilon$.
\end{proof}

\subsection{KL Computation and the Bretagnolle-Huber 
Contradiction} \label{sec:lb_kl_bh}

\begin{proposition}[KL Divergence] \label{prop:lb_kl}
For any algorithm $\pi$, hypothesis 
$\beta \in \{-1,+1\}^N$, and bin $j \in [N]$, let 
$\beta^{\oplus j}$ denote $\beta$ with the $j$-th component 
flipped. Let $N^\pi_j(\beta)$ 
denote the expected number of pulls from bin $j$.  Then:
\begin{equation}
    \mathrm{KL}\left( P^\pi_\beta \,\Big\|\, 
    P^\pi_{\beta^{\oplus j}} \right) \le 
    2\varepsilon^2 N^\pi_j(\beta).
\end{equation}
\end{proposition}
\begin{proof}
The contexts $\{x_{t,a}\}$ are drawn i.i.d.\ 
from $N(0,1)$ independently of $\beta$, and the noise is 
$\eta_t \sim N(0,1)$.  By the chain rule for KL divergence, 
the divergence accumulates solely from the reward 
distributions of the selected arms.  Because $f_\beta$ and 
$f_{\beta^{\oplus j}}$ differ only inside $B_j$, and the 
Gaussian KL divergence for means differing 
by $\Delta\mu$ with unit variance is 
$\frac{1}{2}(\Delta\mu)^2$, with 
$|\Delta\mu| = 2\varepsilon|\psi_j(x)| \le 2\varepsilon$, 
summing over the $N_j^\pi(\beta)$ rounds where the chosen 
arm falls in bin $j$ yields the result.
\end{proof}

\begin{proof}[Proof of Theorem \ref{thm:lower_bound}]
Let $\beta^{(0)} = -\mathbf{1}$ be the \emph{null} 
hypothesis, and $\beta^{(j)} = -\mathbf{1} + 2e_j$ be the 
\emph{spike} hypothesis where bin $j$ is uniquely optimal.  
Define the event $A_j = \{ \#\{t : j_t = j\} \ge T/2 \}$.

Under $\beta^{(j)}$, on the event 
$\mathcal{E} \cap A_j^c$, the algorithm plays sub-optimal 
bins for more than $T/2$ rounds.  By 
Lemma~\ref{lem:lb_per_round_regret}:
\begin{equation} \label{eq:lb_regret_spike}
    \E_{\beta^{(j)}}[R_T] \ge \frac{\varepsilon T}{2} 
    \Prob_{\beta^{(j)}}(\mathcal{E} \cap A_j^c) 
    \ge \frac{\varepsilon T}{2}\left( 
    \Prob_{\beta^{(j)}}(A_j^c) - 2T^{-2/3} \right).
\end{equation}

By Proposition~\ref{prop:lb_kl}, 
$\mathrm{KL}(P_{\beta^{(0)}} \| P_{\beta^{(j)}}) 
\le 2\varepsilon^2 N^\pi_j(\beta^{(0)})$.  Since 
$\sum_j N^\pi_j(\beta^{(0)}) \le T$, Markov's inequality 
ensures there exists $J^\star \subseteq [N]$ with 
$|J^\star| \ge N/2$ such that for all $j \in J^\star$, 
$N^\pi_j(\beta^{(0)}) \le 2T/N$.  For any $j \in J^\star$:
\begin{equation*}
    \mathrm{KL}\left( P_{\beta^{(0)}} \,\Big\|\, 
    P_{\beta^{(j)}} \right) \le 2\varepsilon^2 
    \left( \frac{2T}{N} \right) = \frac{T}{N^3} \le 1,
\end{equation*}
where we substituted $\varepsilon = 1/(2N)$ and 
$N = \lceil T^{1/3} \rceil$.

The Bretagnolle-Huber inequality gives:
\begin{equation} \label{eq:lb_bh}
    \Prob_{\beta^{(0)}}(A_j) + \Prob_{\beta^{(j)}}(A_j^c) 
    \ge \frac{1}{2} \exp\left( 
    - \mathrm{KL}\left( P_{\beta^{(0)}} \,\Big\|\, 
    P_{\beta^{(j)}} \right) \right) \ge \frac{1}{2e}.
\end{equation}

We now conduct a case analysis of $J^\star$:
\begin{itemize}
    \item \textbf{Case A:} Suppose there exists $j \in J^\star$ with 
    $\Prob_{\beta^{(j)}}(A_j^c) \ge \frac{1}{4e}$.  For $T$ 
    large enough that $2T^{-2/3} \le \frac{1}{8e}$, 
    substituting into \eqref{eq:lb_regret_spike}:
    \begin{equation*}
        \E_{\beta^{(j)}}[R_T] \ge \frac{\varepsilon T}{2} \cdot \frac{1}{8e} = \frac{T}{32eN} \ge \frac{T}{32e(2T^{1/3})} = \frac{T^{2/3}}{64e}.
    \end{equation*}
    \item \textbf{Case B:} Suppose for all $j \in J^\star$, 
    $\Prob_{\beta^{(j)}}(A_j^c) < \frac{1}{4e}$.  By 
    \eqref{eq:lb_bh}, $\Prob_{\beta^{(0)}}(A_j) > 
    \frac{1}{4e}$ for all $j \in J^\star$.  Summing:
    \begin{equation*}
        \frac{N}{8e} \le \sum_{j \in J^\star} 
        \Prob_{\beta^{(0)}}(A_j) \le 2,
    \end{equation*}
    which is impossible for $N \ge 16e$, i.e., 
    $T \ge (16e)^3$.
\end{itemize}
Case A must hold.  Taking the supremum over the hypothesis 
class concludes the proof with $c = 1/(64e)$.
\end{proof}


\section{Detailed Experiments} \label{app:experiments}

We present an empirical evaluation of \texttt{ZoomSIB-UCB} against three baselines: {GSTOR} (a \emph{non-monotone} link function based algorithm), {ESTOR} (a \emph{monotone}-link function based algorithm)\cite{kang2025single}, and {Random} (uniform arm selection). We evaluate performance across three synthetic link functions to stress-test shape-agnosticism: (1) \textit{Quadratic} $f(z)=-(z-1)^{2}+1$; (2) \textit{Asymmetric} $f(z)=z\cdot e^{-z^{2}}$; and (3) \textit{Zigzag} (multimodal) $f(z) = \sin(z) + 0.3z$. Unless stated otherwise, $d=10$, contexts are standard Gaussian, $K=20$, and results average over 30 independent trials with shaded bands indicating $\pm 1$ standard deviation.

All experiments were conducted on Google Colab with an NVIDIA T4 GPU (12.67 GB RAM). 
The total compute required was approximately 15 GPU-hours, with the majority of runtime 
consumed by the horizon and dimension scaling experiments in Section~\ref{sec:scaling}.

\paragraph{Adaptive Stopping: no knowledge of $T_0$ (or $\mu^*$) needed.} Theoretically we need to fix $T_0 = \tilde{\mathcal{O}}(d^{2}T^{2/3})$, which may be very large. To resolve 
this, we equip \texttt{ZoomSIB-UCB} with an adaptive stopping rule that monitors the 
empirical drift $\norm{\hat{\theta}_0^{(t)} - \hat{\theta}_0^{(t-\Delta t)}}_1$ every 
$\Delta t = 100$ rounds, terminating Phase~1 as soon as the drift falls below 
$\varepsilon = 0.05$. A hard cap of $T_0 \le 0.4T$ is retained as a safety 
net for low-signal environments. The resulting Auto-Cap metric exposes a 
clear hierarchy of signal strength across our test functions, as shown in Table \ref{tab:exp1}: 
high-signal topologies such as Quadratic and Zigzag trigger termination 
at as little as $1.8\%$ of $T$, while the Asymmetric function---whose 
geometry yields flat zero-gradient tails---correctly prolongs exploration up to $40\%$ at small horizons, with the cap shrinking 
as $T$ grows and more signal accumulates.

\subsection{Scaling Dynamics: Horizon and Dimension}\label{sec:scaling}

We examine scaling dynamics across time horizons 
$T \in \{2000, 5000, 10000, 25000, 40000\}$ (fixed $d=10$) and ambient 
dimensions $d \in \{5, 10, 15, 20, 30\}$ (fixed $T=40000$). As shown in 
Figure~\ref{fig:exp1} and Table~\ref{tab:exp1}, \texttt{ZoomSIB-UCB} achieves strictly 
sublinear regret and dominates both baselines at every horizon across all 
functions. The advantage is starkest on the Quadratic and Zigzag geometries, 
where GSTOR incurs roughly $7\times$ and $10\times$ higher regret, 
respectively, at $T=40\,000$. 

Transitioning to dimension scaling (detailed in Figure~\ref{fig:exp2} and Table~\ref{tab:exp2}), because \texttt{ZoomSIB-UCB} operates entirely in the 
one-dimensional projected index space after the Stein estimation phase, it 
successfully bypasses the curse of dimensionality that handicaps kernel-based 
methods: on the Zigzag function, a $6\times$ increase in dimension ($d=5$ to 
$d=30$) inflates ZoomSIB's regret by only $2.3\times$ (from $\approx 1\,530$ 
to $\approx 3\,475$), while GSTOR's regret nearly doubles on an already much 
higher baseline.

\begin{figure}[htbp]
    \centering
    \begin{subfigure}[b]{\textwidth}
        \centering
        \includegraphics[width=0.5\textwidth, height=0.28\textheight, keepaspectratio, page=1]{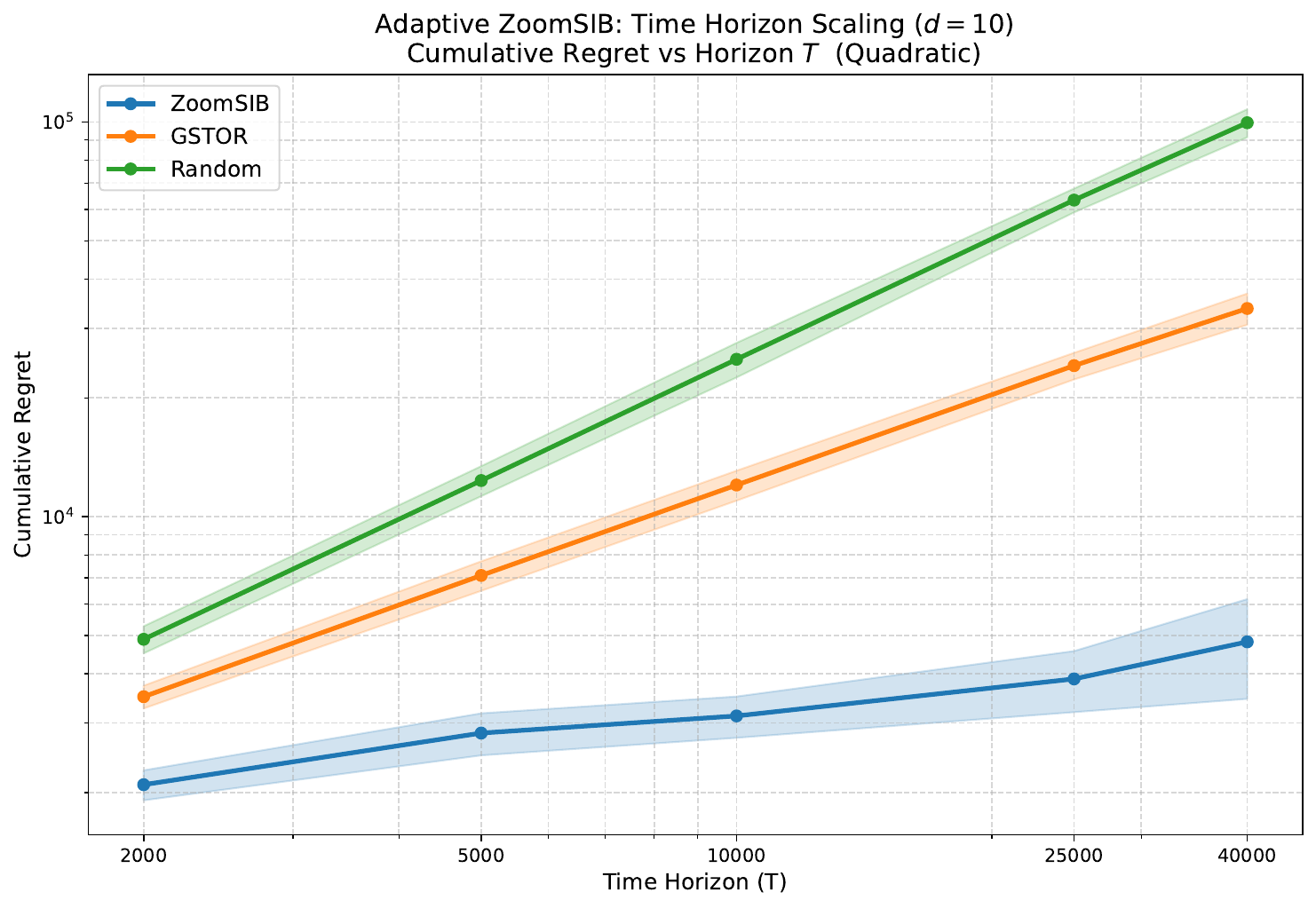}
        \caption{Quadratic}
    \end{subfigure}
    
    \vspace{1em}
    
    \begin{subfigure}[b]{\textwidth}
        \centering
        \includegraphics[width=0.5\textwidth, height=0.28\textheight, keepaspectratio, page=2]{images/adaptive_time_scaling.pdf}
        \caption{Asymmetric}
    \end{subfigure}
    
    \vspace{1em}
    
    \begin{subfigure}[b]{\textwidth}
        \centering
        \includegraphics[width=0.5\textwidth, height=0.28\textheight, keepaspectratio, page=3]{images/adaptive_time_scaling.pdf}
        \caption{Zigzag}
    \end{subfigure}
    \caption{\textbf{Horizon Scaling.} Cumulative regret vs.\ time horizon $T$ ($d=10$). \texttt{ZoomSIB-UCB} achieves strictly sublinear regret, significantly outperforming GSTOR and Random exploration across all geometries.}
    \label{fig:exp1}
\end{figure}

\begin{figure}[htbp]
    \centering
    \begin{subfigure}[b]{\textwidth}
        \centering
        \includegraphics[width=0.5\textwidth, height=0.28\textheight, keepaspectratio, page=1]{images/adaptive_dim_scaling_T40K.pdf}
        \caption{Quadratic}
    \end{subfigure}
    \hfill
    \begin{subfigure}[b]{\textwidth}
        \centering
        \includegraphics[width=0.5\textwidth, height=0.28\textheight, keepaspectratio, page=2]{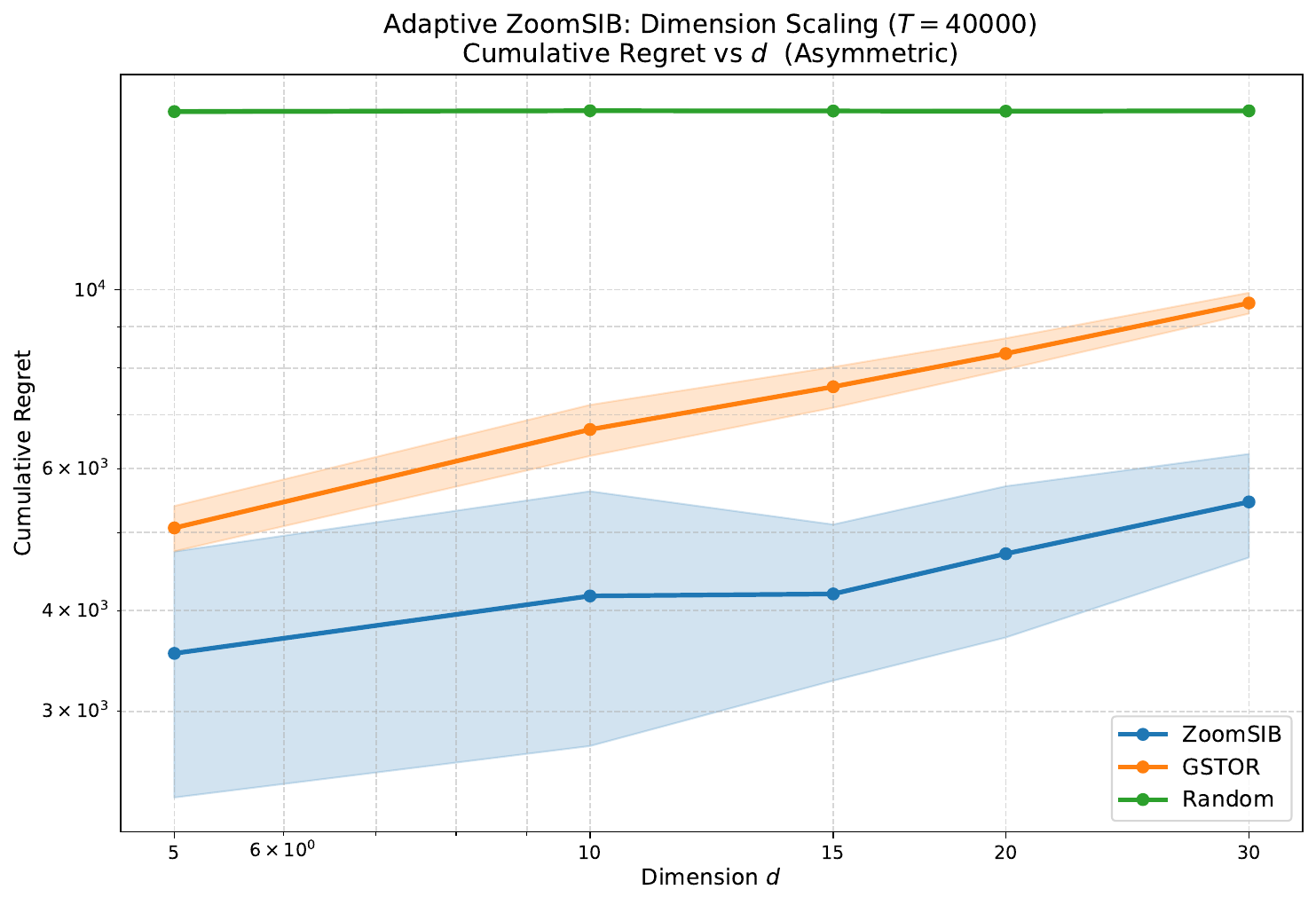}
        \caption{Asymmetric}
    \end{subfigure}
    \hfill
    \begin{subfigure}[b]{\textwidth}
        \centering
        \includegraphics[width=0.5\textwidth, height=0.28\textheight, keepaspectratio, page=3]{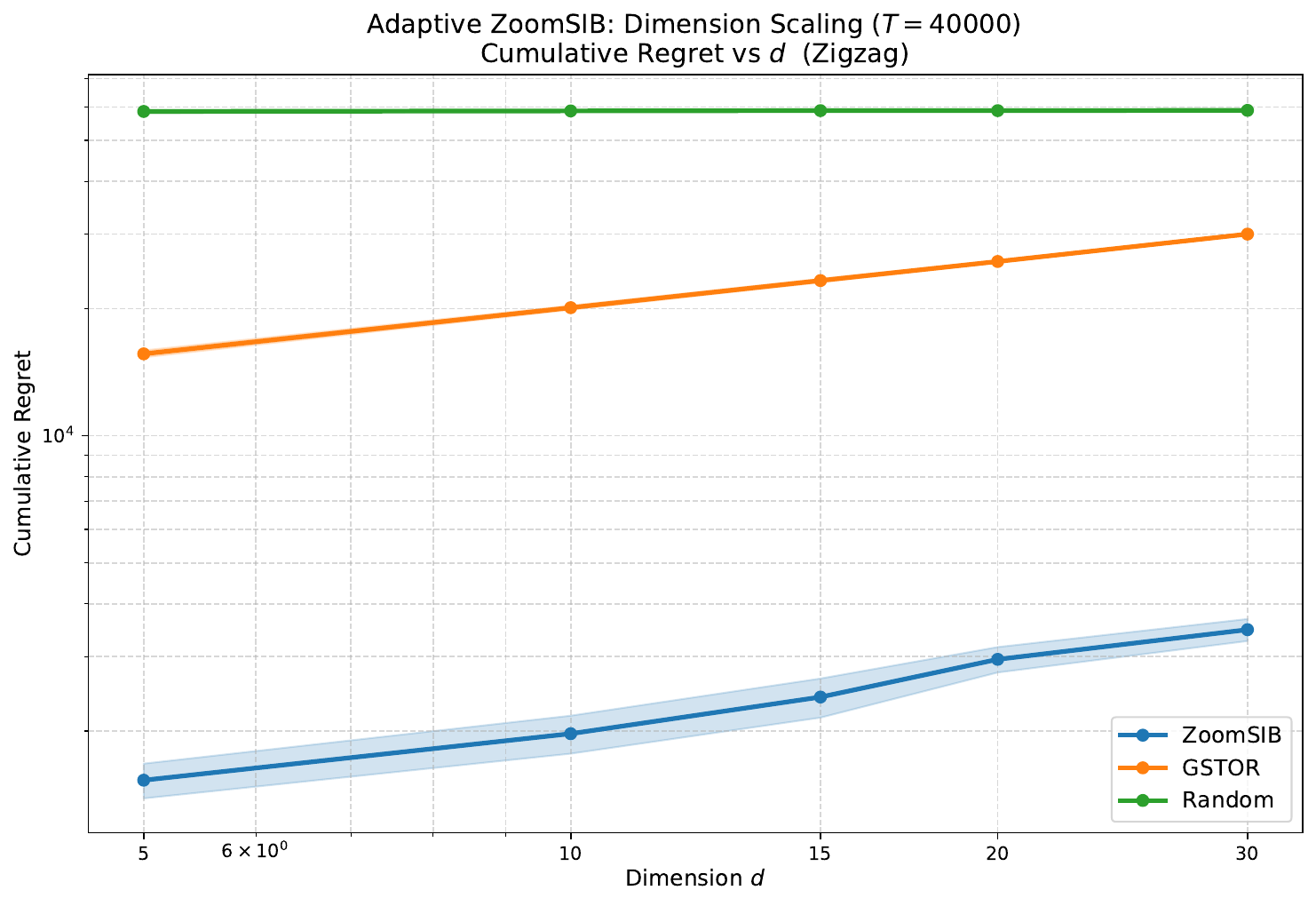}
        \caption{Zigzag}
    \end{subfigure}
    \caption{\textbf{Dimension Scaling.} Cumulative regret vs.\ ambient dimension $d$ at $T=40\,000$. \texttt{ZoomSIB-UCB} operates entirely in the 1D projected space after Phase 1, successfully avoiding the curse of dimensionality.}
    \label{fig:exp2}
\end{figure}

\subsection{Theoretical Verification}

To empirically verify our asymptotic guarantees, we investigate the regret growth rate on the Quadratic, Asymmetric, and Zigzag functions at $d=10$ for $T\in\{2\,000, 5\,000, 10\,000, 25\,000, 40\,000\}$. By fitting a log-log linear regression $\log R_T = \alpha \log T + \text{const}$, we extract the empirical scaling exponent $\alpha$.

Figure \ref{fig:exp3_asym},\ref{fig:exp3_quad}, and \ref{fig:exp3_zig} displays the log-log plots with empirical fits against theoretical reference lines for the asymmetric, quadratic and zigzag functions respectively. The empirical slopes for \texttt{ZoomSIB-UCB} remain securely beneath the theoretical worst-case ceiling of $2/3$ $(\approx 0.667)$. This performance gap is anticipated: the $\Omega(T^{2/3})$ hard instance from Theorem \ref{thm:lower_bound} relies on highly localized adversarial perturbations. On standard smooth functions, the adaptive rule terminates Phase 1 early, reducing the effective exploration penalty. In contrast, GSTOR adheres strictly to an empirical slope near $0.75$, perfectly matching its $T^{3/4}$ bound \cite{kang2025single}, and the Random policy scales linearly at $1.00$.

\begin{figure}[H]
    \centering
    \includegraphics[width=0.5\textwidth, height=0.85\textheight, keepaspectratio, page=1]{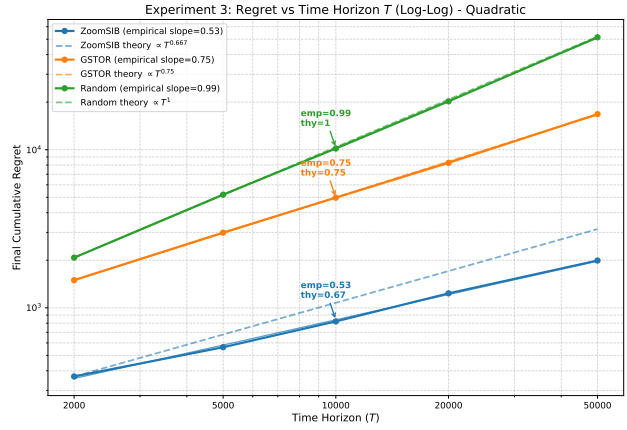}
    \caption{\textbf{Log-Log Regret Scaling (Quadratic).} Empirical (solid) and theoretical reference (dashed) slopes. ZoomSIB achieves an exponent strictly $< 2/3$, confirming its sublinear scaling.}
    \label{fig:exp3_quad}
\end{figure}

\begin{figure}[H]
    \centering
    \includegraphics[width=0.5\textwidth, height=0.85\textheight, keepaspectratio, page=2]{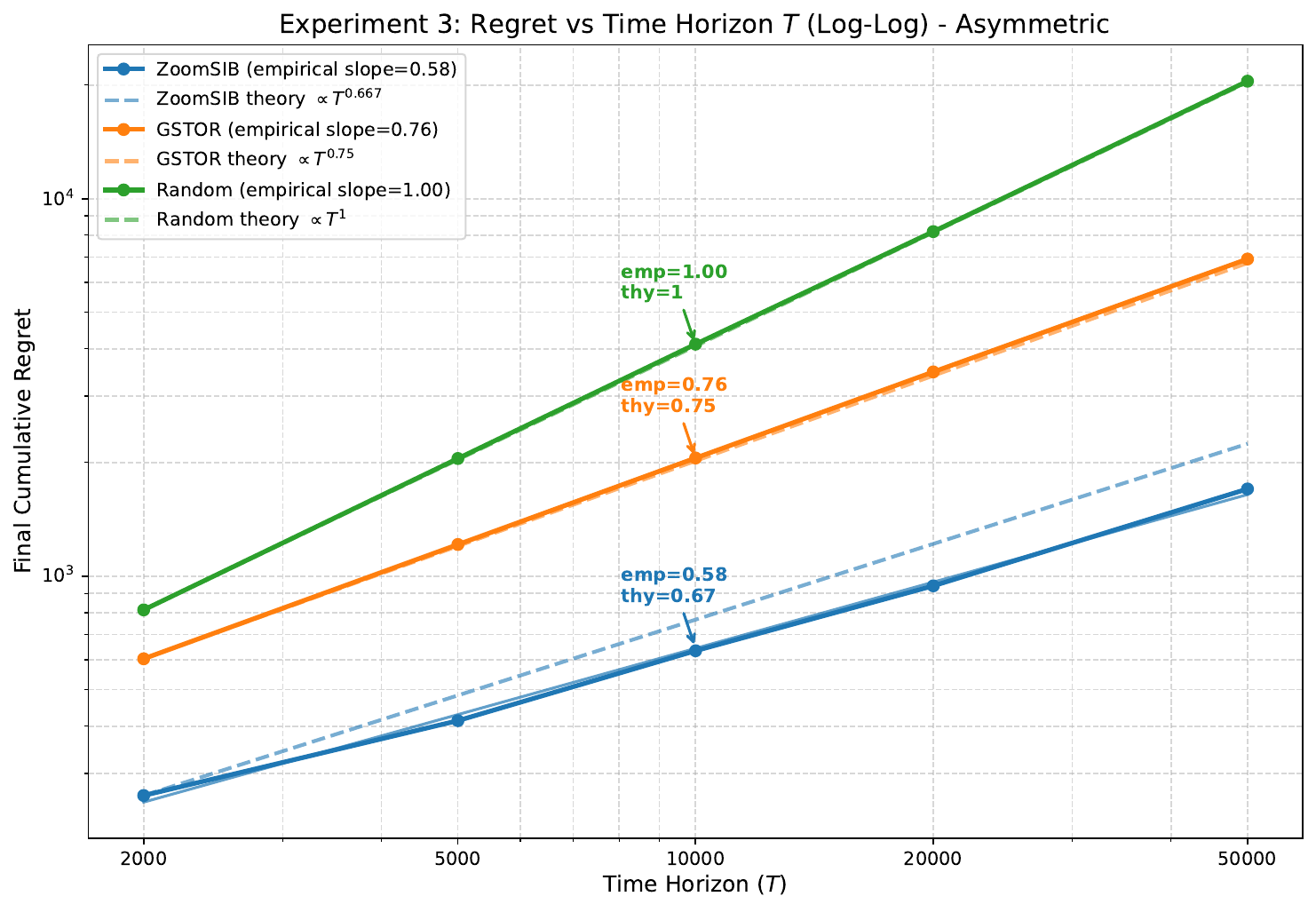}
    \caption{\textbf{Log-Log Regret Scaling (Asymmetric).} Empirical (solid) and theoretical reference (dashed) slopes. ZoomSIB correctly scales sublinearly even with flat, zero-gradient tails.}
    \label{fig:exp3_asym}
\end{figure}

\begin{figure}[H]
    \centering
    \includegraphics[width=0.5\textwidth, height=0.85\textheight, keepaspectratio, page=3]{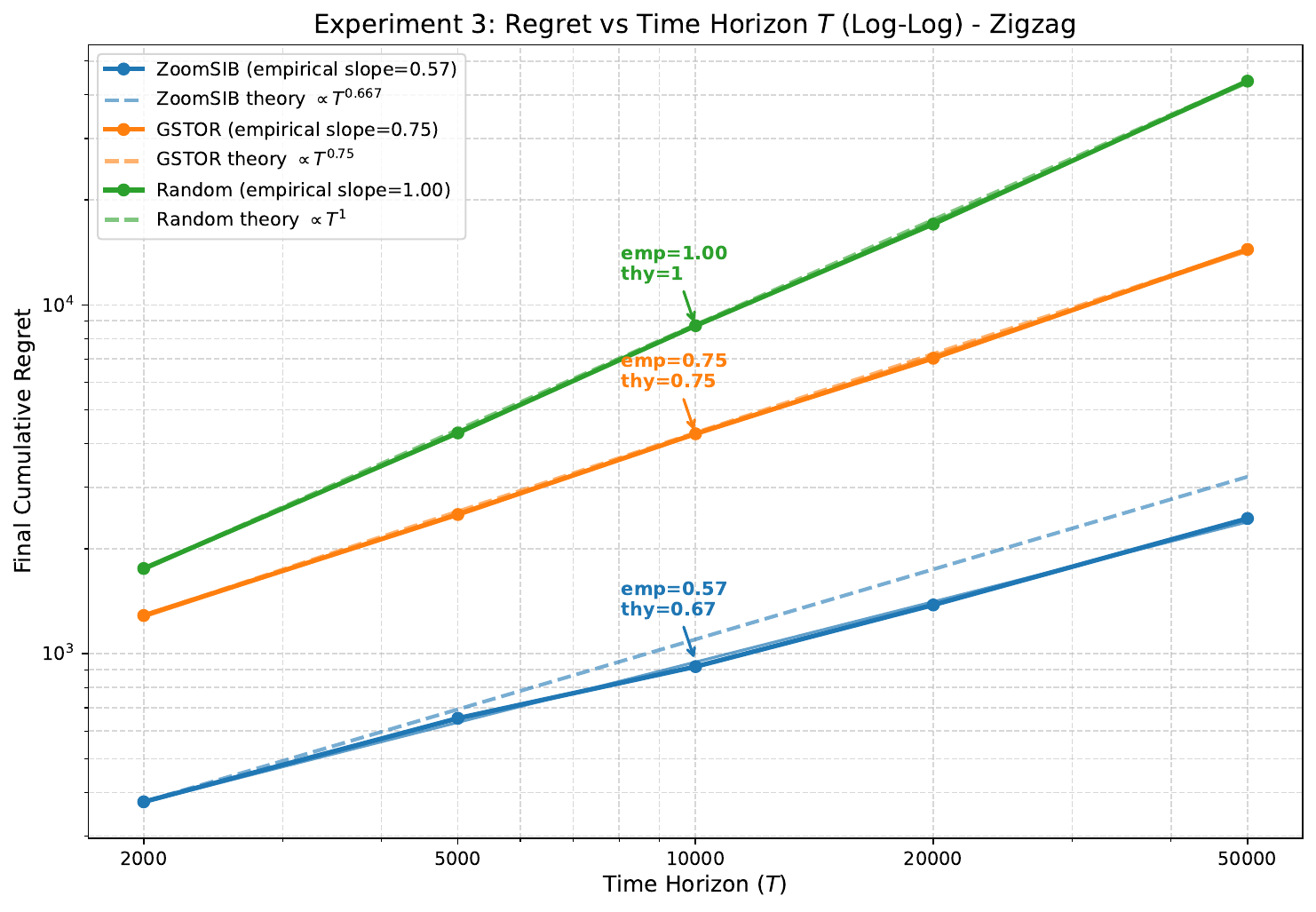}
    \caption{\textbf{Log-Log Regret Scaling (Zigzag).} Empirical (solid) and theoretical reference (dashed) slopes. ZoomSIB easily bypasses the local deceptive optima.}
    \label{fig:exp3_zig}
\end{figure}

\subsection{Real-World Datasets} \label{sec:real_world}

We further validate \texttt{ZoomSIB-UCB} on two standard real-world benchmarks: the KDD Cup 99 dataset (Network Intrusion) and the Forest Cover Type dataset. To convert these offline classification datasets into a sequential multi-armed bandit environment, we extract continuous numerical features ($d=39$ for KDD 99, $d=55$ for Forest Cover) and apply K-Means clustering to partition the data into $K=32$ distinct clusters, which serve as our arms. At each round $t$, a single observation is randomly sampled from each cluster, and the algorithm must select one. The reward is binary: $1$ if the observation belongs to the target class (`normal' traffic for KDD, `Spruce/Fir' for Forest Cover), and $0$ otherwise. Results are averaged over 10 independent runs for $T=10\,000$.

The results expose a severe vulnerability in theoretical exploration bounds when applied to high-dimensional real-world data. Because GSTOR relies on the theoretical formula $T_1 = \mathcal{O}(d^{3/8} T^{3/4})$, its double exploration phase mathematically forces it to blindly explore for $79.0\%$ of the rounds in KDD 99 (avg. $7\,900$ rounds) and a crippling $89.9\%$ of the rounds in Forest Cover (avg. $8\,988$ rounds). As a result, GSTOR incurs massive regret: $3191.30 \pm 59.16$ (KDD 99) and $7242.70 \pm 67.32$ (Forest Cover). 

Conversely, our Adaptive \texttt{ZoomSIB-UCB} dynamically tracks the Stein score and autonomously realizes it has mapped the feature space long before the theoretical worst-case bound. It triggers an early exit at just $21.1\%$ (KDD 99) and $29.5\%$ (Forest Cover) of the horizon, seamlessly transitioning to UCB exploitation. This self-calibrating mechanism allows \texttt{ZoomSIB-UCB} to crush the baseline, achieving a final regret of $1276.80 \pm 205.90$ on KDD 99 and $5579.90 \pm 482.60$ on Forest Cover, demonstrating that adaptive variance-tracking is a strict requirement for deploying single-index bandits in the real world.

\begin{figure}[H]
    \centering
    \begin{subfigure}[b]{\textwidth}
        \centering
        \includegraphics[width=0.5\textwidth, height=0.48\textheight, keepaspectratio]{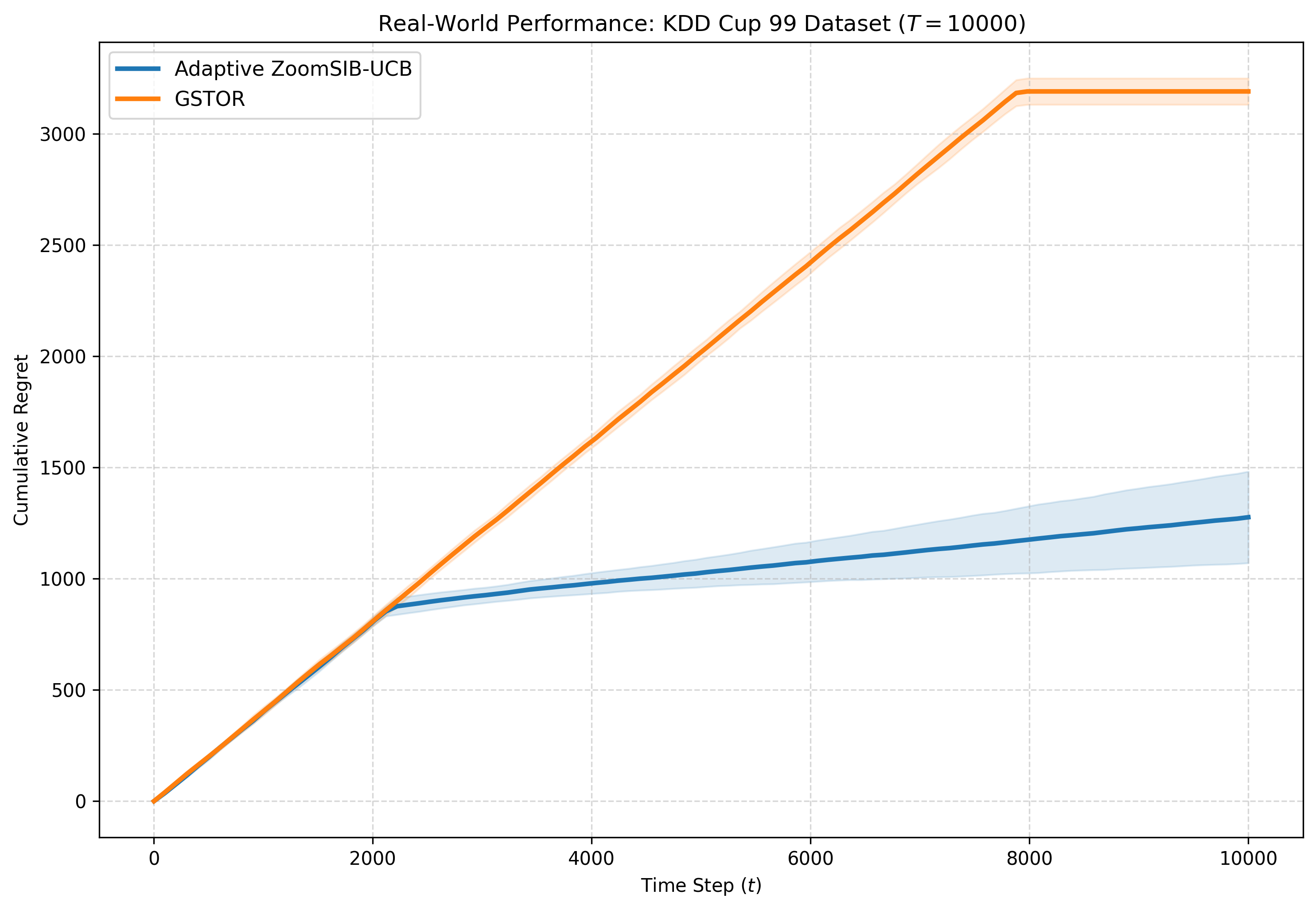}
        \caption{KDD Cup 99 ($d=39, K=32$)}
    \end{subfigure}
    \hfill
    \begin{subfigure}[b]{\textwidth}
        \centering
        \includegraphics[width=0.5\textwidth, height=0.48\textheight, keepaspectratio]{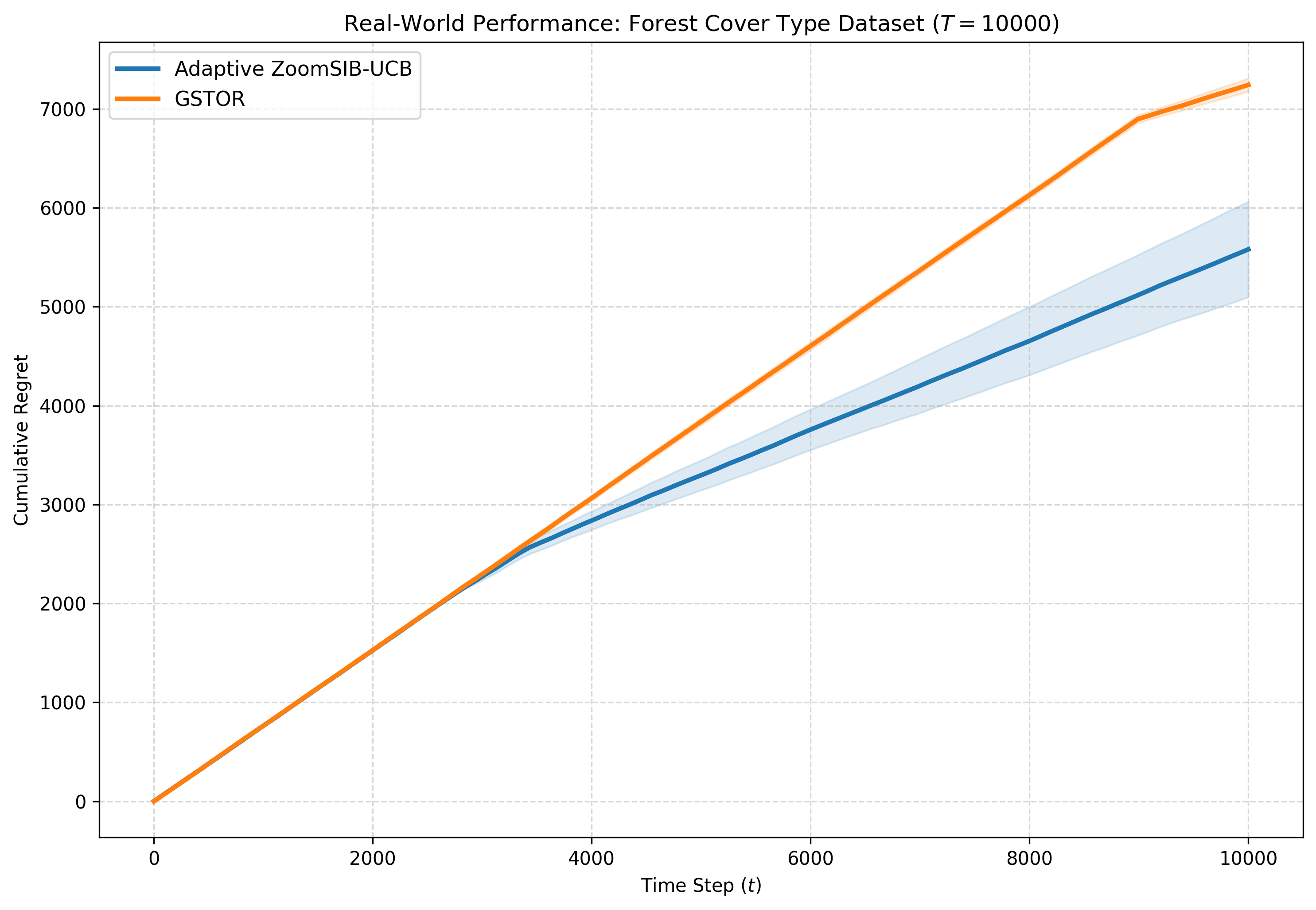}
        \caption{Forest Cover Type ($d=55, K=32$)}
    \end{subfigure}
    \caption{\textbf{Real-World Datasets.} Cumulative regret on offline classification datasets transformed into contextual bandit environments ($T=10\,000$). GSTOR's theoretical constraints force it to over-explore for up to $89.9\%$ of the horizon. In contrast, Adaptive \texttt{ZoomSIB-UCB} autonomously detects signal convergence, exiting exploration early and drastically minimizing regret.}
    \label{fig:real_world}
\end{figure}
\subsection{Robustness to Model Misspecification}

Finally, we test the algorithms in environments where the reward link is fundamentally misspecified for monotone-assumed models. We evaluate two non-monotone environments: $f(z)=\sin(2z)$ and $f(z)=\sin(2z)-0.5z^{2}$. The ESTOR algorithm acts as our misspecified baseline, as it structurally assumes $f$ is non-decreasing. 

As illustrated in Figure \ref{fig:exp4} and Table \ref{tab:exp4}, ESTOR suffers a catastrophic failure under these geometries, collapsing into linear regret growth. On the highly complex asymmetric function $f(z)=\sin(2z)-0.5z^{2}$, ESTOR actually accumulates significantly more regret than the pure Random baseline. Conversely, ZoomSIB's shape-agnostic architecture effortlessly maintains sublinear regret throughout, terminating with $16\,430 \pm 5\,773$ cumulative regret against ESTOR's $67\,667 \pm 6\,814$---a $4.1\times$ reduction despite having access to the exact same data stream. This underscores that ZoomSIB's non-parametric UCB structure over projected bins is not merely a theoretical construct, but a critical necessity for safe deployment in environments where the reward function shape is unknown.

\vspace{0.5em}

\begin{figure}[ht]
  \centering
  \includegraphics[width=0.8\textwidth, height = \textheight, keepaspectratio]{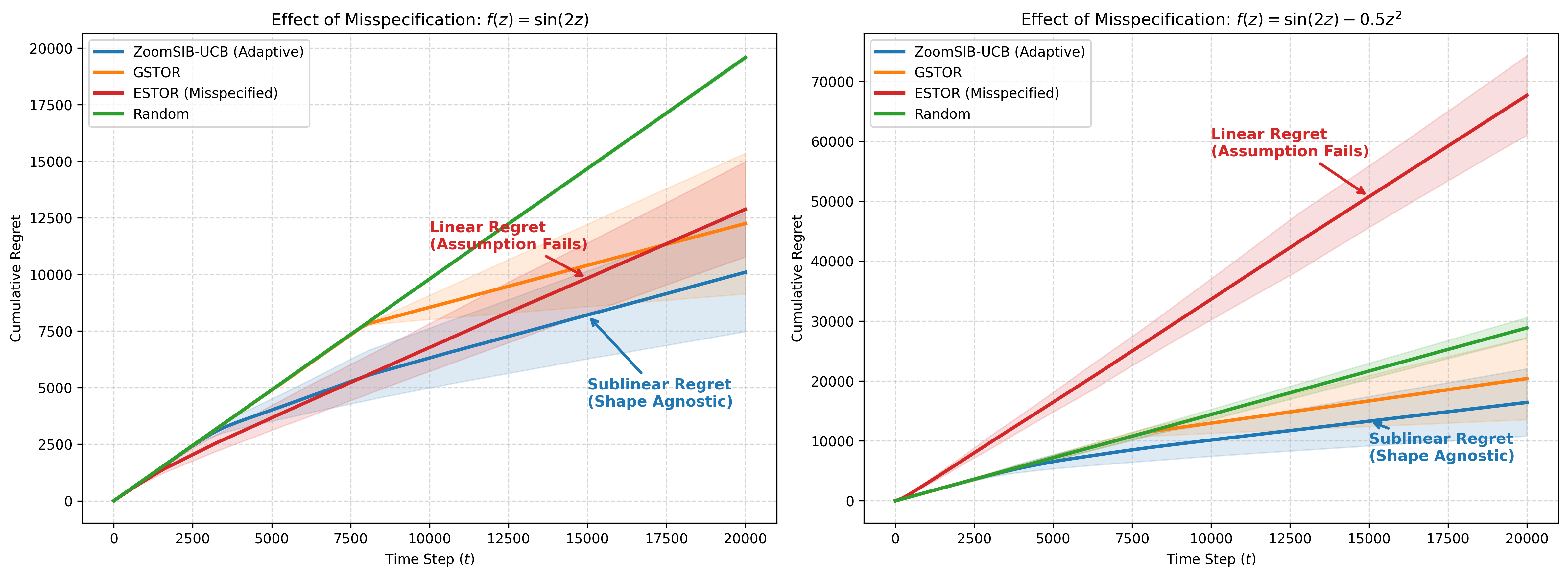}
  \caption{\textbf{Model Misspecification.} Cumulative regret vs.\ time for $f(z)=\sin(2z)$ (left) and $f(z)=\sin(2z)-0.5z^{2}$ (right). ESTOR (red) collapses to linear regret, while \texttt{ZoomSIB-UCB} (blue) safely maintains sublinear performance.}
  \label{fig:exp4}
\end{figure}

\begin{table}[H]
\centering
\caption{%
  \textbf{Horizon Scaling} ($d=10$, mean $\pm$ std, 30 trials).
  \texttt{ZoomSIB-UCB} achieves strictly sublinear cumulative regret across all time horizons
  and all three link functions, consistently outperforming both GSTOR and Random.
}
\label{tab:exp1}
\setlength{\tabcolsep}{4pt}
\renewcommand{\arraystretch}{1.18}
\resizebox{\textwidth}{!}{%
\begin{tabular}{l r r r r r}
\toprule
\rowcolor{headerbg}
\textbf{Function} & $T$ & \textbf{Auto-Cap} & \textbf{ZoomSIB} & \textbf{GSTOR} & \textbf{Random} \\
\midrule
\multirow{5}{*}{Quadratic}
  & 2\,000  & 37.5\% (750)  & \best{2094.04 $\pm$ 182.92}  & \good{3498.31 $\pm$ 232.55}   & 4889.82 $\pm$ 386.31     \\
  & 5\,000  & 16.7\% (835)  & \best{2829.68 $\pm$ 345.38}  & \good{7096.14 $\pm$ 609.67}   & 12344.40 $\pm$ 1079.93   \\
  & 10\,000 & 9.1\% (910)   & \best{3127.09 $\pm$ 375.01}  & \good{12024.97 $\pm$ 1049.10} & 25038.09 $\pm$ 2517.72   \\
  & 25\,000 & 3.4\% (860)   & \best{3880.88 $\pm$ 684.64}  & \good{24130.01 $\pm$ 1871.97} & 63383.17 $\pm$ 4402.28   \\
  & 40\,000 & 2.2\% (865)   & \best{4817.61 $\pm$ 1363.25} & \good{33681.45 $\pm$ 3062.29} & 99524.59 $\pm$ 8138.95   \\
\midrule
\multirow{5}{*}{Asymmetric$^{\dagger}$}
  & 2\,000  & 40.0\% (800)  & \best{546.98 $\pm$ 59.40}    & \good{701.19 $\pm$ 38.31}   & 836.28 $\pm$ 13.61   \\
  & 5\,000  & 34.7\% (1735) & \best{1113.31 $\pm$ 121.85}  & \good{1440.38 $\pm$ 81.77}  & 2092.67 $\pm$ 16.69  \\
  & 10\,000 & 18.0\% (1800) & \best{1620.83 $\pm$ 284.94}  & \good{2457.03 $\pm$ 167.55} & 4173.12 $\pm$ 29.91  \\
  & 25\,000 & 7.2\% (1790)  & \best{2623.32 $\pm$ 734.01}  & \good{4716.67 $\pm$ 257.20} & 10417.17 $\pm$ 48.78 \\
  & 40\,000 & 4.6\% (1825)  & \best{3602.55 $\pm$ 1143.70} & \good{6685.82 $\pm$ 363.28} & 16688.73 $\pm$ 73.85 \\
\midrule
\multirow{5}{*}{Zigzag}
  & 2\,000  & 38.0\% (760)  & \best{1212.16 $\pm$ 74.78}   & \good{2100.81 $\pm$ 47.02}  & 2941.98 $\pm$ 27.75   \\
  & 5\,000  & 14.8\% (740)  & \best{1359.30 $\pm$ 128.93}  & \good{4216.63 $\pm$ 73.79}  & 7347.30 $\pm$ 76.38   \\
  & 10\,000 & 7.7\% (770)   & \best{1470.57 $\pm$ 158.39}  & \good{7154.54 $\pm$ 94.58}  & 14691.33 $\pm$ 121.01 \\
  & 25\,000 & 3.0\% (745)   & \best{1769.33 $\pm$ 252.53}  & \good{14125.67 $\pm$ 177.18}& 36794.13 $\pm$ 207.85 \\
  & 40\,000 & 1.8\% (730)   & \best{2020.43 $\pm$ 246.92}  & \good{20055.38 $\pm$ 261.18}& 58775.07 $\pm$ 283.15 \\
\bottomrule
\end{tabular}%
}
\vspace{0.3em}
\raggedright
\footnotesize{$^{\dagger}$\textit{Note on Auto-Cap:} The Asymmetric function requires a substantially longer Phase 1 exploration at smaller horizons because its geometry features flat, zero-gradient tails, yielding a weaker initial signal. The adaptive rule autonomously detects this and correctly prolongs exploration until the Stein estimator safely converges.}
\end{table}

\begin{table}[ht]
\centering
\caption{\textbf{Dimension Scaling at $T=40\,000$} (mean $\pm$ std, 30 trials). Random omitted.}
\label{tab:exp2}
\setlength{\tabcolsep}{5pt}
\renewcommand{\arraystretch}{1.18}
\resizebox{0.85\textwidth}{!}{%
\begin{tabular}{l r r r r}
\toprule
\rowcolor{headerbg}
\textbf{Function} & $d$ & \textbf{Auto-Cap} & \textbf{ZoomSIB} & \textbf{GSTOR} \\
\midrule
\multirow{5}{*}{Quadratic}
  & 5  & 1.7\% (695)   & \best{4065.17 $\pm$ 1335.59} & 27924.93 $\pm$ 5616.88 \\
  & 10 & 2.1\% (825)   & \best{4115.63 $\pm$ 711.70}  & 32166.26 $\pm$ 1420.25 \\
  & 15 & 2.9\% (1140)  & \best{5369.42 $\pm$ 1076.47} & 40399.23 $\pm$ 4371.87 \\
  & 20 & 3.0\% (1220)  & \best{6348.24 $\pm$ 1274.46} & 44514.44 $\pm$ 4258.43 \\
  & 30 & 4.3\% (1710)  & \best{7512.44 $\pm$ 959.54}  & 51858.21 $\pm$ 3380.79 \\
\midrule
\multirow{5}{*}{Asymmetric}
  & 5  & 2.5\% (1015)  & \best{3539.72 $\pm$ 1194.59} & \good{5065.32 $\pm$ 327.35} \\
  & 10 & 4.3\% (1720)  & \best{4170.52 $\pm$ 1453.64} & \good{6709.79 $\pm$ 488.31} \\
  & 15 & 5.9\% (2350)  & \best{4194.86 $\pm$ 920.30}  & \good{7580.89 $\pm$ 440.49} \\
  & 20 & 7.4\% (2980)  & \best{4705.24 $\pm$ 1000.94} & \good{8334.85 $\pm$ 371.19} \\
  & 30 & 8.8\% (3515)  & \best{5455.79 $\pm$ 801.33}  & \good{9629.01 $\pm$ 287.24} \\
\midrule
\multirow{5}{*}{Zigzag}
  & 5  & 1.6\% (630)   & \best{1530.02 $\pm$ 144.51}  & 15625.01 $\pm$ 291.38 \\
  & 10 & 1.9\% (765)   & \best{1971.39 $\pm$ 203.27}  & 20096.80 $\pm$ 226.58 \\
  & 15 & 2.5\% (985)   & \best{2407.08 $\pm$ 255.51}  & 23296.45 $\pm$ 226.13 \\
  & 20 & 2.9\% (1175)  & \best{2956.26 $\pm$ 204.58}  & 25848.30 $\pm$ 162.75 \\
  & 30 & 4.0\% (1610)  & \best{3475.11 $\pm$ 208.71}  & 30011.30 $\pm$ 170.86 \\
\bottomrule
\end{tabular}%
}
\end{table}

\begin{table}[ht]
\centering
\caption{\textbf{Final cumulative regret at $T=20\,000$} (mean $\pm$ std, 30 trials).}
\label{tab:exp4}
\setlength{\tabcolsep}{5pt}
\renewcommand{\arraystretch}{1.22}
\resizebox{\textwidth}{!}{%
\begin{tabular}{l r r r r r}
\toprule
\rowcolor{headerbg}
\textbf{Function} & \textbf{Auto-Cap} & \textbf{ZoomSIB} & \textbf{GSTOR} & \textbf{ESTOR\,(misspec.)} & \textbf{Random} \\
\midrule
$f(z)=\sin(2z)$
  & 15.0\% (3005)
  & \best{10097 $\pm$ 2691}
  & \good{12246 $\pm$ 3181}
  & \bad{12874 $\pm$ 2154}
  & 19580 $\pm$ 87   \\[4pt]
$f(z)=\sin(2z)-0.5z^{2}$
  & 20.1\% (4010)
  & \best{16430 $\pm$ 5773}
  & \good{20406 $\pm$ 7027}
  & \bad{67667 $\pm$ 6814}
  & 28857 $\pm$ 1779 \\
\bottomrule
\end{tabular}%
}
\end{table}

\newpage

\end{document}